\documentclass[twoside]{article}

\usepackage[accepted]{aistats2022}
\usepackage[round]{natbib}

\usepackage{graphicx}			
\usepackage{amssymb}
\usepackage{amsmath}
\usepackage{amsthm}
\usepackage{bbm}
\usepackage{xcolor}
\newtheorem{proposition}{Proposition}[section]
\newtheorem{theorem}{Theorem}[section]
\newtheorem{assumption}{Assumption}[section]
\newtheorem{remark}{Remark}[section]
\newtheorem{corollary}{Corollary}[section]
\newtheorem{lemma}{Lemma}[section]

\graphicspath{{graphs/}}

\newcommand{\cN}{{\mathcal{N}}}

\newcommand{\cX}{{\mathcal{X}}}
\newcommand{\cY}{{\mathcal{Y}}}
\newcommand{\bV}{{\mathbf{V}}}
\newcommand{\bW}{{\mathbf{W}}}
\newcommand{\bX}{{\mathbf{X}}}
\newcommand{\bY}{{\mathbf{Y}}}
\newcommand{\Reals}{{\mathbb{R}}}

\begin{document}

%
\runningtitle{A Class of Geometric Structures in Transfer Learning}

%
\runningauthor{Zhang, Blanchet, Ghosh, Squillante}

\twocolumn[

\aistatstitle{A Class of Geometric Structures in Transfer Learning: \\
Minimax Bounds and Optimality}

\aistatsauthor{ Xuhui Zhang \And Jose Blanchet \And  Soumyadip Ghosh \And Mark S.\ Squillante }

\aistatsaddress{ Stanford University \And  Stanford University \And IBM Research \And IBM Research } ]

\begin{abstract}
 We study the problem of transfer learning, observing that previous efforts to understand its information-theoretic limits do not fully exploit the geometric structure of the source and target domains. In contrast, our study first illustrates
the benefits of incorporating a natural geometric structure within a linear regression model, which corresponds to the generalized eigenvalue problem formed by the Gram matrices of both domains. We next establish a finite-sample minimax lower bound, propose a refined model interpolation estimator that enjoys a matching upper bound, and then extend our framework to multiple source domains and generalized linear models.
Surprisingly, as long as information is available on the distance between the source and target parameters, negative-transfer does not occur. Simulation studies show that our proposed interpolation estimator outperforms state-of-the-art transfer learning methods in both moderate- and high-dimensional settings.
\end{abstract}

\section{INTRODUCTION}
The task of transferring knowledge from one 
domain (source) to another related domain (target)
is known as transfer learning. This task arises naturally in a wide range of applications where data is scarce in the target domain but substantial in a source domain believed to be somewhat similar to the target. For example, in the context of marketing and demand prediction for products in a new market, it is natural to use source information involving well-tested markets. Similarly, demand prediction for new products can be estimated using source information from current market products~\citep{ref:AFRIN2018data}. 
Rigorous statistical formulations of transfer learning introduce non-trivial challenges. This includes balancing the tension between tractability in the training procedure and flexibility in order to reflect the differences between the source and target environments. In addition to this tension, a useful modeling framework should provide statistical insights on the efficiency gain induced by introducing source information into target inference.

Our goal in this paper is to introduce a transfer learning formulation under linear and generalized linear models that addresses the above tractability-flexibility tension and produces effective efficiency insights. In particular, the contributions of this paper include:

($i$) Our formulation provides an easy-to-compute transfer learning estimator that optimally (in a precise sense) interpolates the target and source parameters subject to an uncertainty region which controls the differences between source and target models.

($ii$) Our modelling framework exposes a natural geometric structure that is built on using the Fisher information metric (also known as information geometry) which we exploit in order to understand the main drivers of transfer learning from source to target.

($iii$) We are able to provide a finite-sample minimax lower bound and show that the worst-case risk of our estimator in ($i$) achieves (up to a constant) the minimax lower bound uniformly over the magnitude of the difference between the target and source models.

One of the insights from our formulation, for example, is that as long as information is available on the distance between the source and target parameters, negative-transfer does not occur. Namely, the worst-case risk of our estimator in ($i$) is always smaller than the minimax risk of using the target dataset alone.

Although there is a significant amount of literature on applied transfer learning procedures \citep{ref:PanYang,ref:torreyshavlik2010,ref:WeissEtAl,ref:taskesen2021sequential}, the literature on rigorous mathematical formulations that lead to minimax optimal estimators is limited. 
\cite{ref:cai2021minimax} consider transfer learning in the context of a stylized non-parametric classification setting under a different set of assumptions including a model that only allows posterior drift.
\cite{ref:Kpotufe2021} study classification settings similar to that of~\cite{ref:cai2021minimax}, albeit under covariate-shift assumptions on the difference in source and target environments.
In strong contrast, while our methodology is parametric, our analysis is not limited to classification problems and it further supports more general environments including more general drift conditions between the target and source models, where indeed our analysis of the underlying geometry explicitly accounts for these differences.

\cite{ref:bastani2021predicting} considers linear and non-linear regression models similar to our work, and proposes a two-step joint estimator for transfer learning. However, their focus is on high-dimensional settings under sparsity assumptions and no minimax optimality result is established; moreover, since they use an $l_1$ norm for the analysis, their results are not directly comparable to our results.
\cite{ref:LiCaiLi2020} extend
the work of \cite{ref:bastani2021predicting} to allow multiple source domains and establish a minimax lower bound for high-dimensional linear regression models (LRMs). However, their bound is asymptotic in the number of target samples, and they further constrain the difference between the source and target parameters relative to the size of the target sample. ~\cite{ref:tian2021transfer} further extend this work of \cite{ref:LiCaiLi2020} to high-dimensional generalized linear models (GLMs).
In strong contrast, we study a fixed-difference 
environment without any such constraint
and the performance of our optimal estimator is measured in terms of finite-sample bounds which match (up to a computable constant factor) the minimax lower bound uniformly over the magnitude of the difference between the models. 

\cite{ref:Kalan2020} consider the LRM setting and involve the spectral gap of the generalized eigenvalue problem we consider, with definitions for the population distribution of their random-design setting analogous to ours in the fixed-design setting.
However, in strong contrast, our analysis handles the entire spectrum of the generalized eigenvalues and our results essentially provide a tighter lower bound in comparison with the results of \cite{ref:Kalan2020}, thus illustrating the significance of our geometric perspective.

A simulation study compares
our estimator to the work of~\cite{ref:bastani2021predicting},~\cite{ref:LiCaiLi2020} and
more direct methods commonly used in practice (e.g., pooling all available data). As
our estimator is designed in a fixed-dimension setting, the simulation results confirm the strong performance of our approach 
in moderate dimensions. Moreover, a heuristic modification of our methods using an $l_1$ regularization, as in \cite{ref:LiCaiLi2020} and \cite{ ref:bastani2021predicting}, seems to provide further improvements over these methods in
the sparse high-dimensional case. We plan to investigate these types of modifications as part of our future research.
Finally, results based on a real-world dataset further confirm the strong performance of our approach.

To briefly summarize our approach, let us consider the linear regression transfer learning problem between a source and a target with Gaussian errors. We start the construction of our estimator by reparameterizing the linear regression estimators. Instead of expressing them in terms of the canonical bases, as is customary, we express them in terms of a generalized eigenvalue problem that arises when computing the Fisher-Rao distance using the design matrices of the source and target. This distance induces a Riemmanian geometric structure between the models. Specifically, convex combinations in the reparameterized (Riemmanian) space correspond to general
interpolations in the (original space of) LRM parameters.

Our estimator is obtained from the class generated by the convex combination of estimators for the source and target in the reparameterized space. Then, we obtain the minimax estimator by maximizing over models within a given distance while minimizing over our chosen class of estimators. The minimization is carried out over convex combination parameters which are different for each coefficient. The final estimator is transformed back to the canonical basis.

For our minimax lower bound, we rely once again on the Fisher information metric and the reparameterization used in the design of our estimator. We then apply Le Cam's two-point minimax method which in our case reduces to studying two carefully designed hypotheses for each coefficient.

Section~\ref{sec:LRM} presents
our mathematical framework and theoretical results for transfer learning under
LRMs.
Then,
Section~\ref{sec:GLM} extends
our
framework to support a class of 
GLMs
as well as multiple sources.
Section~\ref{sec:sim} presents simulation results showing that our algorithm outperforms various transfer learning methods.
The supplement contains additional
theoretical and simulation 
results, related technical details, and all proofs.

\section{LINEAR REGRESSION MODELS}
\label{sec:LRM}
We consider our mathematical framework for transfer learning within the context of
LRMs.
We introduce our framework and establish our theoretical results on minimax bounds in Section~\ref{ssec:base}, where we also propose a refined model interpolation estimator that is minimax optimal. Then, in Section~\ref{ssec:compare}, we compare the minimax bounds to the basic approaches discussed in~\cite{ref:Daume07}, showing the latter to be suboptimal.


\subsection{Mathematical Framework}
\label{ssec:base}
Let $\cX$ denote a $d$-dimensional feature space, $\cY$ a response space, 
and $\cN(\mu,\sigma^2)$ the normal distribution with mean $\mu$ and variance $\sigma^2$.
Our base source LRM can then be formally written as
\begin{equation}
y_i = x_i^\top \theta_S + \epsilon_i , \qquad i \in [n_S] ,
\label{eq:base:source}
\end{equation}
where $\theta_S \in \Reals^d$ is the regression coefficient for the source model, $n_S$ is the number of samples from the source model, 
$x_i\in\cX$ are (fixed) designs and $y_i\in\cY$ are independent response samples,
$\epsilon_i \sim \cN(0,\sigma_S^2)$ are independent noise random variables for $i\in [n_S]$, and $[n] := \{1,\ldots,n\}$. 
Similarly, our base target LRM can be formally written as
\begin{equation}
v_i = w_i^\top \theta_T + \eta_i , \qquad i \in [n_T] ,
\label{eq:base:target}
\end{equation}
where $\theta_T \in \Reals^d$ is the regression coefficient for the target model, $n_T$ is the number of samples from the target model, 
$w_i\in\cX$ are (fixed) designs and $v_i\in\cY$ are independent response samples, and $\eta_i \sim \cN(0,\sigma_T^2)$ are independent noise random variables for $i\in [n_T]$. We denote the distribution of the models~\eqref{eq:base:source} and~\eqref{eq:base:target} by $P_S$ and $P_T$, respectively.

For ease of exposition, we collect the designs and responses in~\eqref{eq:base:source} from the source domain into the design matrix $\bX$ and the response vector $\bY$,
respectively; i.e.,
$x_i^\top$ is the $i$-th row of $\bX$ and $y_i$ is the $i$-th element of $\bY$. Similarly,
we collect the
target
designs and responses in~\eqref{eq:base:target}
into the design matrix $\bW$ and the response vector $\bV$, respectively.
With $\epsilon := (\epsilon_i)_{i\in [n_S]}$ and $\eta := (\eta_i)_{i\in [n_T]}$,
we can then write the
LRM
as
\begin{align}
\mathbf{Y} &= \mathbf{X} \theta_S + \epsilon, \qquad \epsilon_i\sim\mathcal{N}(0,\sigma_S^2), \label{eq:sourcelinearregmodel} \\
\mathbf{V} &= \mathbf{W} \theta_T +\eta, \qquad \eta_i\sim\mathcal{N}(0,\sigma^2_T).
\label{eq:targetlinearregmodel}
\end{align}
We consider fixed-design matrices to be (arbitrarily) different for the source and target domains, thus focusing on a combination of concept drift and a version of covariate shift best suited to the fixed-design setting, which is similar to \cite{ref:bastani2021predicting}.

Our interests lie in estimating the regression coefficient $\theta_T$ for the target domain. Hence, we assume that the noise variance $\sigma_T^2$ is known and then we consider the family of distributions~\eqref{eq:targetlinearregmodel} parameterized by $\theta_T$ as a manifold. Various geometries can be defined on a manifold of statistical models~\citep{ref:nielsen2020elmentary}, among which the Fisher-Riemannian manifold given by the Fisher information metric tensor is particularly useful, representing the
unique invariant metric tensor under Markov embeddings up to a scaling constant~\citep{ref:ampbell1986AnE}. For the Gaussian location model~\eqref{eq:targetlinearregmodel}, we can compute the Fisher information matrix with respect to $\theta_T$,
up to a scaling constant,
as $\mathbf{W}^\top\mathbf{W}$. 

Given an estimate $\widehat{\theta}$ of $\theta_T$, we consider the Riemannian geodesic metric distance (or the Fisher-Rao distance) as a principled way to measure the dissimilarity of $\widehat{\theta}$ to the (unknown) ground truth $\theta_T$. We therefore define the loss function $\ell(\widehat{\theta},\theta_T)$ as
\begin{equation}\label{eq:lossfunction}
\ell(\widehat{\theta} , \theta_T) = (\widehat{\theta} - \theta_T)^\top (\mathbf{W}^\top\mathbf{W}) (\widehat{\theta} - \theta_T).
\end{equation}
Such a loss function was used and termed ``prediction loss'' in~\cite{ref:lee2020minimax} without the above geometrical motivation. However, as we will show, the geometric structure of the source and target models can play an important role in transfer learning, in particular, the discrepancy between the Fisher-Rao distances induced by the models~\eqref{eq:sourcelinearregmodel} and~\eqref{eq:targetlinearregmodel}.

We aim to employ minimax theory to establish the optimality of statistical learning procedures. Given an estimator $\widehat{\theta}$ arising from any learning procedure, we consider its worst-case risk over an uncertainty set of plausible distributions; refer to~\citep[Chapter~2]{ref:tsybakov2008introduction}. In the transfer learning setting, one natural uncertainty set is given by all source and target distributions whose dissimilarity is upper-bounded. However, instead of the $l_0$ or $l_1$-norm typically used in the high-dimensional setting~\citep{ref:tian2021transfer,ref:bastani2021predicting,ref:LiCaiLi2020}, we characterize the dissimilarity between $\theta_S$ and $\theta_T$ in terms of the Fisher-Rao distance induced by~\eqref{eq:targetlinearregmodel}, and thus the transfer learning uncertainty set is given by $\{D(\theta_S,\theta_T)\leq U^2\}$ where
\begin{equation}\label{eq:uncertaintysetmetric}
D(\theta_S,\theta_T) = (\theta_S - \theta_T)^\top(\mathbf{W}^\top\mathbf{W}) (\theta_S-\theta_T) .
\end{equation}
While our approach is based on knowledge of $U$, we note that requirements of knowing certain population quantities to rigorously prove optimality are also prevalent in recent studies such as \cite{ref:bastani2021predicting}, \cite{ref:cai2021minimax},
\cite{ref:LiCaiLi2020},
\cite{ref:tian2021transfer}.

We further remark that the form of $D(\theta_S,\theta_T)$ in~\eqref{eq:uncertaintysetmetric} is chosen for the convenience of our analysis, though more general forms are available to us. In fact, for our analysis to go through, it suffices to choose 
\[
D(\theta_S,\theta_T) = (\theta_S - \theta_T)^\top \mathbf{O}(\theta_S-\theta_T),
\]
for $\mathbf{O}$ positive-definite, and where $(\bW^\top\bW)^{-1}\mathbf{O}$ commutes with $(\bW^\top\bW)^{-1}\bX^\top\bX$.

These ingredients lead to the minimax risk formulation
\begin{equation}\label{eq:minimaxformulation}
R = \inf_{\widehat{\theta}}\sup_{D(\theta_S,\theta_T)\leq U^2} \mathbb{E}_{P_S,P_T}[\ell(\widehat{\theta},\theta_T)] .
\end{equation}
Our minimax risk takes the infimum over all estimators $\widehat{\theta}$ given the data and takes the supremum over all pairs $(\theta_S,\theta_T)$ whose distance \eqref{eq:uncertaintysetmetric} is bounded by $U^2$.
Fixing some $\theta_S$ implies that samples from the source domain are not useful, whereas we establish minimax bounds with the finiteness of both source and target samples playing a role, consistent with the existing literature.

In order to determine $R$, we obtain an upper bound $B$ and a lower bound
$L$ on $R$. Note that the maximum risk of any estimator provides an upper bound. Hence, we first construct a novel estimator that utilizes the geometric structure of the source and target models, and then derive a matching lower bound (up to a constant factor) using Le Cam's method~\citep[Chapter~2]{ref:tsybakov2008introduction}. We make the following assumption throughout the rest of Section~\ref{ssec:base}.
\begin{assumption}\label{a:pdw}
The Gram matrix $\bW^\top\bW$ corresponding to the target model is positive-definite.
\end{assumption}

\subsubsection{Derivation of the Upper Bound}
\label{sec:upperboundlr}
The key to obtain a tight upper bound $B$ is to note that the geometric structural difference between the Fisher-Riemannian manifolds induced by the models~\eqref{eq:sourcelinearregmodel} and~\eqref{eq:targetlinearregmodel} is fully described by the generalized eigenvalue problem of the
pencil $(\mathbf{X}^\top\mathbf{X},\mathbf{W}^\top\mathbf{W})$; refer to~\cite{ref:GolubVanLoan}.
Specifically, we have
\[
\mathbf{X}^\top\mathbf{X} e_i = \lambda_i \mathbf{W}^\top\mathbf{W} e_i ,
\]
with eigenvalues $\lambda_i$ arranged in descending order and eigenvectors $E = (e_1,\ldots,e_d)$ normalized so that
\[
E^\top (\mathbf{W}^\top\mathbf{W}) E = I, \quad E^\top (\mathbf{X}^\top\mathbf{X}) E = \mathrm{diag}(\lambda_1,\ldots,\lambda_d) .
\]
These eigenvectors represent the directions of the spread of designs with the corresponding eigenvalues representing the relative magnitude of the spread in these directions. The generalized eigenvalue problem has been
previously
used to define suitable loss functions for positive definite matrices, such as the F\"{o}rstner metric~\citep{ref:Forstner2003metric} and Stein's loss~\citep{ref:James1992estimation}.

Next, let us write $\theta_S$ and $\theta_T$ in the eigenbasis $E$ as
\begin{equation}\label{eq:coordinatetransform}
\theta_S = E \beta_S, \qquad \theta_T = E\beta_T,
\end{equation}
respectively, and thus the
problem is given by
\begin{align}
\mathbf{Y} &= (\mathbf{X} E)\beta_S + \epsilon, \qquad \epsilon_i\sim\mathcal{N}(0,\sigma_S^2)\label{eq:sourcemodelaftercoordinatetrans} , \\
\mathbf{V} &= (\mathbf{W} E) \beta_T +\eta, \qquad\eta_i\sim\mathcal{N}(0,\sigma^2_T) ,\label{eq:targetmodelaftercoordinatetrans} 
\\
&\ell(\widehat\theta,\theta_T)
=
\tilde\ell(\widehat{\beta},\beta_T) = \|\widehat{\beta}-\beta_T\|_2^2 , \nonumber \\
&D(\theta_S,\theta_T)
=
\tilde D(\beta_S,\beta_T) = \|\beta_S-\beta_T\|_2^2\notag,
\end{align}
where $\widehat{\theta} = E\widehat\beta$. Hence, it is
more
convenient to work with the following reparameterization of the original formulation in \eqref{eq:minimaxformulation}:
\begin{equation}
\inf_{\widehat{\beta}}\sup_{\tilde D(\beta_S,\beta_T)\leq U^2} \mathbb{E}_{P_S,P_T}[\tilde \ell(\widehat{\beta},\beta_T)] .
\label{eq:reparameterization}
\end{equation}

Denote by $\widehat{\beta}_S$ and $\widehat{\beta}_T$ the ordinary least squares estimate for problems~\eqref{eq:sourcemodelaftercoordinatetrans} and~\eqref{eq:targetmodelaftercoordinatetrans}, respectively; namely,
\begin{align*}
\widehat{\beta}_S &= (E^\top\mathbf{X}^\top\mathbf{X} E)^{-1} E^\top\mathbf{X}^\top \mathbf{Y},\\
\widehat{\beta}_T &= (E^\top\mathbf{W}^\top\mathbf{W} E)^{-1} E^\top\mathbf{W}^\top \mathbf{V}.
\end{align*}
Our proposed model averaging estimator then interpolates $\widehat\beta_S$ and $\widehat\beta_T$ coordinate-wise, i.e., we have
\begin{align}
\widehat{\theta}_{t_1,\ldots,t_d} & = E\widehat{\beta}_{t_1,\ldots,t_d},\notag\\
\widehat{\beta}_{t_1,\ldots,t_d} &= \mathrm{diag}(t_1,\ldots,t_d) \widehat{\beta}_S\label{eq:interpolationscheme}\\
&\quad + \mathrm{diag}(1-t_1,\ldots,1-t_d)\widehat{\beta}_T, \qquad t_i\in[0,1].\notag
\end{align}

We have the following main result for an upper bound $B$ on problem~\eqref{eq:minimaxformulation}, or equivalently on problem~\eqref{eq:reparameterization}.
\begin{theorem}
Under Assumption~\ref{a:pdw}, an upper bound $B$ is given by 
\begin{align}
\inf_{t_1,\ldots,t_d}\sup_{D(\theta_S,\theta_T)\leq U^2} &\mathbb{E}_{P_S,P_T}[\ell(\widehat{\theta}_{t_1,\ldots,t_d},\theta_T)]\label{eq:ubfdproblem}\\
&\qquad =   \sum_{i=1}^d \frac{1}{\frac{1}{\sigma_T^2} + \frac{1}{\alpha_i^\star U^2+  \frac{\sigma_S^2}{\lambda_i}}},\label{eq:ubfdformula}
\end{align}
where 
\[
\alpha_i^\star = \begin{cases}
 \sum_{j=i}^{K^\star}\kappa_j + \frac{1}{K^\star+1} (1-\sum_{j=1}^{K^\star}j\kappa_j) & \textrm{if } i\leq K^\star+1,\\
  0 & \textrm{if } i > K^\star + 1,
\end{cases}
\]
\[
K^\star =\max_{\sum_{i=1}^K i\kappa_i\leq1,0\leq K\leq d-1}K,
\]
\[
\kappa_i = \frac{\sigma_S^2}{U^2}(\frac{1}{\lambda_{i+1}}-\frac{1}{\lambda_i}), \qquad i= 1,\ldots,d-1 .
\]
Moreover, the optimal estimator $\widehat{\theta}_{t_1^\star,\ldots,t_d^\star}$ satisfies 
\begin{equation}\label{eq:optinterpscheme}
t_i^\star = \frac{\sigma_T^2}{\sigma_T^2 + \alpha_i^\star U^2+  \frac{\sigma_S^2}{\lambda_i}}.
\end{equation}
\label{prop:UB:LRM}
\end{theorem}

\textit{Sketch of Proof}: Problem~\eqref{eq:ubfdproblem} can be reformulated into the following finite-dimensional optimization:
\[
\inf_{t_i\in[0,1]}\sup_{\alpha_i\geq 0,\sum_{i=1}^d\alpha_i=1}\sum_{i=1}^d t_i^2\left(\frac{\sigma_S^2}{\lambda_i} + \alpha_i U^2\right) + (1-t_i)^2 \sigma_T^2.
\]
Since the objective function is convex in $t_i$ and concave in $\alpha_i$, by Sion's minimax theorem~\citep{ref:sion1958minimax} we can swap the infimum and supremum to obtain (and, moreover, a pair of Nash equilibrium exists for)
\[
\sup_{\alpha_i\geq 0,\sum_{i=1}^d\alpha_i=1}\inf_{t_i\in[0,1]}\sum_{i=1}^d t_i^2\left(\frac{\sigma_S^2}{\lambda_i} + \alpha_i U^2\right) + (1-t_i)^2 \sigma_T^2.
\]
The inner problem is quadratic and easy to solve, and thus we arrive at 
\[
\sup_{\alpha_i\geq 0,\sum_{i=1}^d\alpha_i=1} \sum_{i=1}^d \frac{1}{\frac{1}{\sigma_T^2} + \frac{1}{\alpha_iU^2 + \frac{\sigma_S^2}{\lambda_i}}},
\]
which again has an explicit solution.\hfill$\square$

\subsubsection{Derivation of the Lower Bound}
\label{sec:lowerboundlr}
Utilizing the coordinate transformation~\eqref{eq:coordinatetransform} and results from Theorem~\ref{prop:UB:LRM}, we next have the following main result for a lower bound $L$ on problem \eqref{eq:minimaxformulation}, or equivalently on problem~\eqref{eq:reparameterization}.
\begin{theorem}
Under Assumption~\ref{a:pdw}, a lower bound $L$ is given by 
\begin{align}
\inf_{\widehat{\theta}}\sup_{D(\theta_S,\theta_T)\leq U^2}&\mathbb{E}_{P_S,P_T}[\ell(\widehat\theta,\theta_T)]\notag\\
&\geq\frac{\exp{\left(-\frac{1}{2}\right)}}{16}\sum_{i=1}^d \frac{1}{\frac{1}{\sigma_T^2} + \frac{1}{\alpha_i^\star U^2+  \frac{\sigma_S^2}{\lambda_i}}}.\label{eq:lbformulalinearregression}
\end{align}
\label{prop:LB:LRM}
\end{theorem}

\textit{Sketch of Proof}: It is more convenient to work with the reparametrization~\eqref{eq:reparameterization}, which is lower bounded by
\begin{align}
    & \inf_{\widehat{\beta}}\sup_{((\beta_S)_i-(\beta_T)_i)^2\leq \alpha_i^\star U^2\forall i}\mathbb{E}_{P_S,P_T}[\tilde\ell(\widehat\beta,\beta_T)]\notag\\
    &\geq \sum_{i=1}^d \inf_{\widehat{\beta}_i}\sup_{((\beta_S)_i-(\beta_T)_i)^2\leq \alpha_i^\star U^2}\mathbb{E}_{P_S,P_T}[(\widehat\beta_i-(\beta_T)_i)^2]\label{eq:lbonedimproblemlinreg}.
\end{align}
We note that $(E^\top\bX^\top \bY)_i$ and $(E^\top\bW^\top\bV)_i$ are sufficient statistics for $(\beta_S)_i$ and $(\beta_T)_i$, respectively,
and
\begin{align*}
(E^\top\bX^\top \bY)_i & = \lambda_i(\beta_S)_i + \tilde\epsilon_i,\quad\tilde\epsilon_i \sim\mathcal{N}(0, \lambda_i\sigma_S^2),\\
(E^\top\bW^\top\bV)_i &= (\beta_T)_i +\tilde\eta_i,\quad\tilde\eta_i\sim\mathcal{N}(0,\sigma_T^2),
\end{align*}
where the noise $\tilde\epsilon_i$ and $\tilde\eta_i$ are independent. For each of the $d$ one-dimensional minimax problems in~\eqref{eq:lbonedimproblemlinreg}, we reduce the problem to the testing of two carefully constructed hypotheses via Le Cam's method. \hfill$\square$

From Theorems~\ref{prop:UB:LRM} and~\ref{prop:LB:LRM}, we observe that the upper bound $B$ and the lower bound $L$ differ by only a constant factor (i.e., $\exp{\left(-\frac{1}{2}\right)}/16$). We therefore have established that the minimax risk $R$ obeys the rate
\[
R \sim \sum_{i=1}^d \frac{1}{\frac{1}{\sigma_T^2} + \frac{1}{\alpha_i^\star U^2+  \frac{\sigma_S^2}{\lambda_i}}}.
\]
Under mild conditions, we obtain that the Gram matrix $\bW^\top\bW$ grows on the order $O_p(n_T)$, and then the minimax risk for the usual $l_2$ loss has the rate
\[
\inf_{\widehat{\theta}}\sup_{D(\theta_S,\theta_T)\leq U^2}\mathbb{E}_{P_S,P_T}[\|\widehat{\theta}-\theta_T\|_2^2] \sim   \sum_{i=1}^d \frac{\frac{1}{n_T}}{\frac{1}{\sigma_T^2} + \frac{1}{\alpha_i^\star U^2+  \frac{\sigma_S^2}{\lambda_i}}}
\]
with probability one on the realization of designs.

\begin{remark}\label{rm:informtheory}
Using the channel capacity of a non-Gaussian additive noise channel~\citep{ref:IHARA1978capacity}, we can improve the uniform constant $\exp{\left(-\frac{1}{2}\right)}/16$ to
$\max\{\exp{\left(-\frac{1}{2}\right)}/16,\left((\sigma_S^2/\lambda_i)/(\alpha_i^\star U^2+\sigma_S^2/\lambda_i)\right)^2\}$
for the $i$-th summand in~\eqref{eq:lbformulalinearregression}. Note that the second term is $1$ if $\alpha_i^\star=0$, and it is arbitrarily close to $1$ if $U$ is sufficiently small.
Details are given in the supplement.
\end{remark}

\begin{remark}
The analysis in \cite{ref:Kalan2020} (in random design) involves the spectral gap of the generalized eigenvalue problem, while our analysis (in fixed design) takes care of the entire spectrum of the generalized eigenvalues. Details are given in the supplement.
\end{remark}


\subsection{Comparison with Basic Approaches}
\label{ssec:compare}
For an analytical comparison with our theoretical results above, we now consider a corresponding mathematical analysis of three basic approaches for transfer learning
often deployed in practice:
use only the source dataset;
use only the target dataset; and
pool both the source and target datasets to train a single model as discussed in~\cite{ref:Daume07}.
We then compare and discuss the results for each of these basic transfer learning approaches with those above in Theorems~\ref{prop:UB:LRM} and \ref{prop:LB:LRM}.

Our theoretical results for the three basic approaches are summarized in the following proposition.

%
%
\begin{proposition}\label{prop:basicapproches}

1.
For the
LRM
based solely on the source dataset, the estimator $\widehat{\theta}_S=E\widehat{\beta}_S$ satisfies
    \[
        \sup_{D(\theta_S,\theta_T)\leq U^2}\mathbb{E}_{P_S,P_T}[\ell(\widehat{\theta}_S,\theta_T)]  = U^2  + \sigma_S^2 \sum_{i=1}^d \lambda_i^{-1} .
    \]
2.
For the
LRM
based solely on the target dataset, the estimator $\widehat{\theta}_T= E\widehat{\beta}_T$ satisfies
    \[
    \sup_{D(\theta_S,\theta_T)\leq U^2}\mathbb{E}_{P_S,P_T}[\ell(\widehat{\theta}_T,\theta_T)]  =  d \sigma_T^2 .
    \]
3.
Finally, for the
LRM
based on pooling the source and target datasets, the estimator 
    \begin{equation}
    \widehat{\theta}_P = \left(\begin{pmatrix}
     \mathbf{X}^\top \;\; \mathbf{W}^\top\end{pmatrix}\begin{pmatrix}
    \mathbf{X}\\\mathbf{W}\end{pmatrix}\right)^{-1} \cdot \begin{pmatrix}
     \mathbf{X}^\top \;\; \mathbf{W}^\top\end{pmatrix}\begin{pmatrix} \mathbf{Y}\\\mathbf{V}\end{pmatrix}
     \label{eq:pooling-estimator}
    \end{equation}
    satisfies
    \begin{align*}
    & \sup_{D(\theta_S,\theta_T)\leq U^2}\mathbb{E}_{P_S,P_T}[\ell(\widehat{\theta}_P,\theta_T)]
    =   U^2\max_{1\leq i\leq d}\left\{\left(\frac{\lambda_i}{1+\lambda_i}\right)^2\right\}
    \\
    &\qquad\qquad\quad
    + \sigma_T^2\sum_{i=1}^d \left(\frac{1}{1+\lambda_i}\right)^2
    + \sigma_S^2 \sum_{i=1}^d \frac{\lambda_i}{(1+\lambda_i)^2} .
    \end{align*}
\end{proposition}
%
%

We observe that the worst-case risk~\eqref{eq:ubfdformula} attained by our proposed estimator is smaller than that of the basic approaches using only the source or target dataset. Moreover, the ordinary least squares estimate $\widehat{\theta}_T$ is minimax optimal when using only the target dataset~\citep[Theorem~6.5]{ref:hodges1950some}. Thus we show, surprisingly, that negative transfer cannot occur if we have access to the bound $U$ on the distance between the source and target parameters. 
In particular, negative transfer does not occur if there exists an estimator such that the worst-case risk of the estimator is smaller than the minimax risk of only using the samples from the target domain, noting that our model interpolation estimator satisfies this condition.
We also show in the supplement that the worst-case risk~\eqref{eq:ubfdformula} is smaller than that of the pooling method. 

\section{GENERALIZED LINEAR MODEL}
\label{sec:GLM}
Our mathematical framework and theoretical results above have been limited to the case of transfer learning with respect to
LRMs.
We next turn to consider our corresponding mathematical framework and theoretical results for the case of transfer learning within the context of a class of 
GLMs.
We
also extend our framework to allow multiple source domains.
The class of GLMs of interest is presented first, followed by our mathematical analysis that leads to theoretical results for GLMs analogous to Theorems~\ref{prop:LB:LRM} and~\ref{prop:UB:LRM}.

Suppose we have access to $M$ different source distributions. For each source $m\in [M]$, assume the $n_{S_m}$ samples $y_1^{(m)},\ldots,y_{n_{S_m}}^{(m)}$ come from the GLM density (with respect to some dominating measure $\mu_{S_m}$)
\begin{align*}
&f^{(m)}(y_i^{(m)};\theta_{S_m}) = \\ &t^{(m)}(y_i^{(m)}) \exp\left(\frac{y_i^{(m)}\langle x^{(m)}_i, \theta_{S_m}\rangle -\Psi^{(m)}(\langle x_i^{(m)},\theta_{S_m}\rangle)}{a^{(m)}(\sigma_{S_m})}\right) ,
\end{align*}
where $(x_i^{(m)})^\top$ is the $i$-th row of the design matrix $\mathbf{X}^{(m)}$, $t^{(m)}(\cdot)$ is a nonnegative-valued function defined on the response space, $a^{(m)}(\cdot)$ is a positive function of $\sigma_{S_m}$, and $\Psi^{(m)}(\cdot)$ is the log-partition function. In a similar manner for the target, assume the $n_T$ samples $v_1,\ldots,v_{n_T}$ come from the GLM density (with respect to a possibly different dominating measure $\mu_T$)
\[
g(v_i;\theta_T) = l(v_i) \exp\left(\frac{v_i\langle w_i, \theta_T\rangle -\Gamma(\langle w_i,\theta_T\rangle)}{b(\sigma_T)}\right) ,
\]
where $w_i^\top$ is the $i$-th row of the design matrix $\mathbf{W}$, $~l(\cdot)$ is analogous to $t^{(m)}(\cdot)$, $~b(\cdot)$ is analogous to $a^{(m)}(\cdot)$, and $\Gamma(\cdot)$ is analogous to $\Psi^{(m)}(\cdot)$.

Note that the above GLM, under which samples are generated according to an exponential family with natural parameter equal to a linear transformation of the underlying parameter $\theta$, was used in~\cite{ref:lee2020minimax} though not in the context of transfer learning.
We further assume that 
\[
\sup_z(\Psi^{(m)})^{''}(z)\leq C_{S_m}~\forall m, \qquad \sup_z\Gamma^{''}(z)\leq C_T .
\]

Following in a manner similar to Section~\ref{ssec:base}, we next can respectively define the loss function $\ell(\widehat{\theta},\theta_T)$ and the uncertainty set $\{D(\theta_{S_m},\theta_T)\leq U_m^2\}$ as
\begin{align*}
\ell(\widehat{\theta} , \theta_T) &= (\widehat{\theta} - \theta_T)^\top (\mathbf{W}^\top\mathbf{W}) (\widehat{\theta} - \theta_T), \\
D(\theta_{S_m},\theta_T) &= (\theta_{S_m} - \theta_T)^\top(\mathbf{W}^\top\mathbf{W}) (\theta_{S_m}-\theta_T) .
\end{align*}
This leads to the problem formulation
\begin{equation}
\inf_{\widehat{\theta}}\sup_{D(\theta_{S_m},\theta_T)\leq U_m^2, \forall m} \mathbb{E}_{P_T,P_{S_m},m\in [M]}[\ell(\widehat{\theta},\theta_T)] .
\label{eq:GLMproblem}
\end{equation}
To simplify the notation, when no confusion arises, we shall henceforth abbreviate the expectation $\mathbb{E}_{P_T,P_{S_m},m\in [M]}[\ell(\widehat{\theta},\theta_T)]$ by $\mathbb{E}[\ell(\widehat{\theta},\theta_T)]$.

Towards solving this problem formulation, consider the generalized eigenvalue problem of the matrix pencil $((\mathbf{X}^{(m)})^\top\mathbf{X}^{(m)},\mathbf{W}^\top\mathbf{W})$; refer to \cite{ref:GolubVanLoan}.  More specifically, we have
\[
(\mathbf{X}^{(m)})^\top\mathbf{X}^{(m)} e^{(m)}_i = \lambda_i^{(m)} \mathbf{W}^\top\mathbf{W} e_i^{(m)} ,
\]
where the eigenvalues are arranged such that $\lambda_1^{(m)}\geq\cdots\geq\lambda_d^{(m)}$. 
Observe that, for any $\theta\in\mathbb{R}^d$, 
\[
\theta^\top (\mathbf{X}^{(m)})^\top\mathbf{X}^{(m)}\theta \leq \lambda_1^{(m)}\theta^\top \mathbf{W}^\top\mathbf{W}\theta .
\]

We then have the following main result for a lower bound on the solution of \eqref{eq:GLMproblem} using Le Cam's and Fano's methods~\citep[Chapter~2]{ref:tsybakov2008introduction}.
\begin{theorem}
A lower bound of the minimax risk corresponding to the GLMs is given by
\begin{align*}
\inf_{\widehat{\theta}}&\sup_{D(\theta_{S_m},\theta_T)\leq U_m^2, \forall m} \mathbb{E}[\ell(\widehat{\theta},\theta_T)] \\
&\qquad\qquad \geq\frac{e^{-1}}{800} \frac{d}{\sum_{m=1}^M\frac{1}{\frac{U_m^2}{d} +\frac{a^{(m)}(\sigma_{S_m})}{C_{S_m}\lambda_1^{(m)}}}+\frac{C_T}{b(\sigma_T)}} .
\end{align*}
\label{thm:glmlb}
\end{theorem}

Details are given in the supplement.

\begin{remark}\label{rm:leeandcourtadelb}
For the non-transfer learning setting considered in~\cite{ref:lee2020minimax}, our proof method gives rise to a lower bound of $d\cdot b(\sigma_T)/C_T$ which is sharper than their lower bound of
\[
\max\left\{\frac{\|\Lambda_\bW\|_1^2}{\|\Lambda_\bW\|_2^2},\lambda_{\textrm{min}}(\bW^\top\bW)\|\Lambda^{-1}_\bW\|_1\right\}\cdot b(\sigma_T)/C_T,
\]
where $\Lambda_\bW$ is the vector of eigenvalues of the positive-definite matrix $\bW^\top\bW$, and $\Lambda_\bW^{-1}$ denotes its coordinate-wise inverse.
\end{remark}


Specializing Theorem~\ref{thm:glmlb} to the LRM case, we derive a multiple sources analog to Theorem~\ref{prop:LB:LRM}.
\begin{corollary}\label{cor:multiplelr}
When the GLMs considered are Gaussian 
LRMs,
then
a lower bound of the minimax risk is
\begin{align*}
\inf_{\widehat{\theta}}&\sup_{D(\theta_{S_m},\theta_T)\leq U_m^2, \forall m} \mathbb{E}[\ell(\widehat{\theta},\theta_T)] \\
&\qquad\qquad\qquad \geq\frac{e^{-1}}{800} \frac{d}{\sum_{m=1}^M\frac{1}{\frac{U_m^2}{d} +\frac{\sigma_{S_m}^2}{\lambda_1^{(m)}}}+\frac{1}{\sigma_T^2}} .
\end{align*}
\end{corollary}
In comparison to Theorem~\ref{prop:LB:LRM}, the rate in Corollary~\ref{cor:multiplelr} involves the spectral gap $\lambda^{(m)}_1 / \lambda^{(m)}_d$.

Turning to consider an upper bound
within the context of GLMs,
we assume that $\inf_z(\Psi^{(m)})^{''}(z)\geq L_{S_m},~\forall m,$ and $\inf_z\Gamma^{''}(z)\geq L_T$. 
Then, using the sub-Gaussian concentration bound for GLM noise and results of \cite{ref:bastani2021predicting}, we obtain the following upper bound of the minimax risk corresponding to the GLMs (ignoring the leading constant):
\begin{equation}\label{glmub}
 \frac{d}{\sum_{m=1}^M \frac{1}{\frac{U_m^2}{d}+\frac{2C_{S_m}a^{(m)}(\sigma_{S_m})}{L_{S_m}^2\lambda_d^{(m)}}}+\frac{1}{\frac{2C_{T}b(\sigma_{T})}{L_{T}^2}}}.
\end{equation}
The details are provided in the supplement.
Comparing \eqref{glmub} with Theorem~\ref{thm:glmlb}, we observe that our upper and lower bounds match up to the ratios $C_{S_m}/L_{S_m}$ and $C_{T}/L_T$ and the spectral gap $\lambda_1^{(m)}/\lambda_d^{(m)}$. We plan to consider the tight analysis of upper and lower bounds in the GLM setting as part of future research.

\section{SIMULATION RESULTS}
\label{sec:sim}
Our focus in this paper is on a mathematical framework of transfer learning and corresponding theoretical results related to geometric structures, minimax bounds, and minimax optimality. To provide further insights and understanding with respect to our framework and results, we now present a collection of simulation results that investigate the quantitative performance of our model interpolation estimator under various conditions, environments, and parameter settings.
These simulation results showcase the ability of our proposed estimator to outperform the basic transfer learning approach discussed in~\cite{ref:Daume07} and the recent state-of-the-art transfer learning methods of~\cite{ref:bastani2021predicting} and~\cite{ref:LiCaiLi2020}.

We consider the LRM in the transfer learning setting of a single source domain and a single target domain.
The optimal interpolation scheme~\eqref{eq:optinterpscheme} requires that we specify the parameters $U,\sigma_S$ and $\sigma_T$, which are typically unknown a priori. In Section~\ref{ssec:simmisspec}, we first qualitatively explore the behavior of our proposed method with respect to the setting of $U$ relative to its true value, while assuming perfect knowledge of the remaining parameters. Then, in Section~\ref{ssec:simcompare}, we treat all three parameters as unknown and develop heuristic estimation procedures with which we compare this full-fledged version of our proposed method against other competing methods in the research  literature.\footnote{All problems are modeled in Python and run on an Intel i5 CPU (1.4GHz) computer.}

\begin{figure}[h]
    \centering
    \includegraphics[width=\columnwidth]{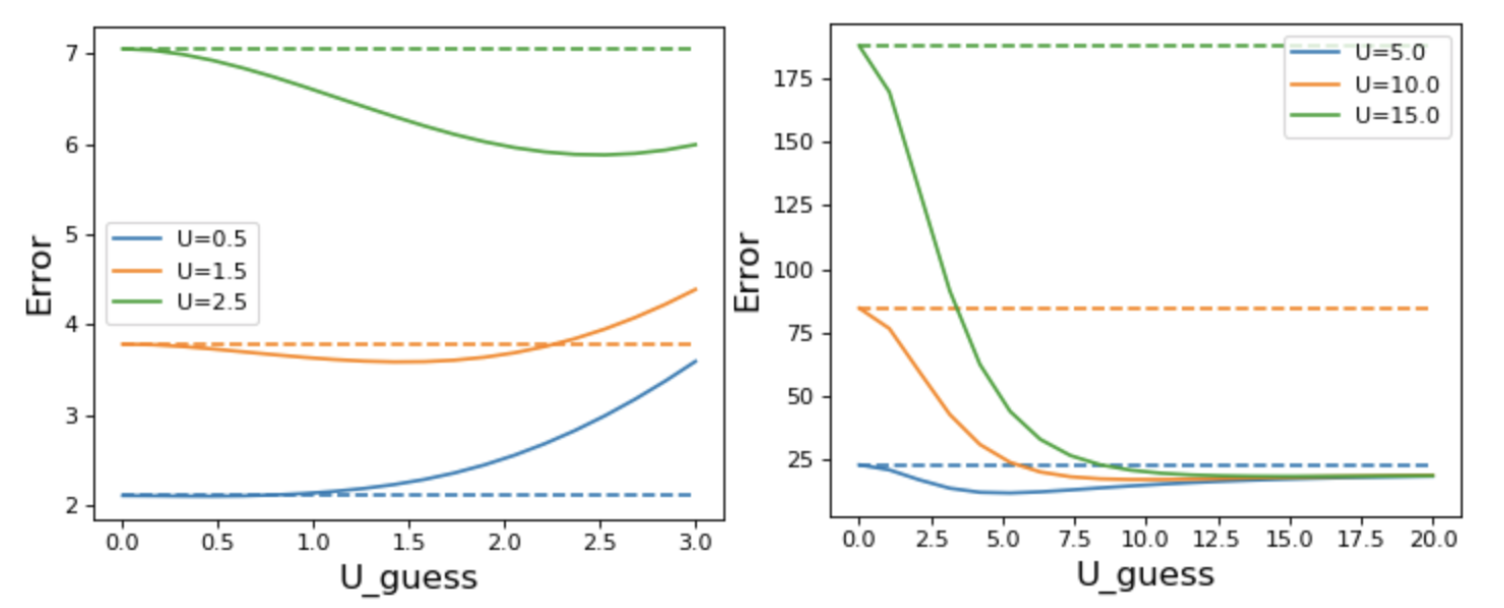}
    \caption{Simulation Results with Different Ground Truth $U$. Solid Lines Represent the Proposed Method. Dashed Lines Represent the Basic Pooling Method.}
    \label{fig:misbaseandlargeu}
\end{figure}




\begin{table*}[tb]
\caption{Simulation Results Comparing the Proposed Method to Other Competing Methods in Moderate-Dimensions (Left-Half) and in High-Dimensions (Right-Half).
``Basic'' Represents the Lowest of the Errors Attained by the Three Basic Methods in Section~\ref{ssec:compare}. Numbers in Parentheses Are Standard Deviations.}
\label{tab:comparisonboth}
\begin{center}
\begin{tabular}{|c|c|c|c|c||c|c|c|c|}
\hline
$U$ &Basic&Proposed&Two-Step&Trans-Lasso &Basic&Proposed&Two-Step&Trans-Lasso \\
\hline
0.5 & 1.9(0.6) & 2.3(1.3) & 3.0(2.5) & \textbf{0.9}(1.8) & \textbf{9.1}(1.7) & 9.3(2.1) & 11.8(3.9) & 9.6(3.0)\\
1.5 & 3.9(1.1) & 4.3(1.5) & 5.7(3.2) & \textbf{2.8}(2.9) & 11.4(2.0) & \textbf{11.4}(2.3) & 14.5(4.6) & 12.2(2.9)\\
2.5 & 6.9(1.5) & 6.5(1.7) & 7.9(2.6) & \textbf{5.3}(2.3) & 15.2(1.9) & \textbf{14.6}(2.5) & 18.1(3.8) & 16.2(2.3)\\
3.5 & 11.7(1.9) & \textbf{9.2}(2.5) & 11.9(3.5) & 9.5(2.9) & 20.5(2.5) & \textbf{18.6}(2.7) & 23.8(3.4) & 21.7(2.8)\\
4.5 & 18.6(2.7) & \textbf{12.4}(4.1) & 15.4(4.9) & 13.9(4.0) & 27.6(2.3) & \textbf{23.6}(3.9) & 31.2(3.1) & 29.2(2.9)\\

5.5 & 20.3(6.1) & \textbf{14.8}(5.3) & 17.9(5.9) & 18.0(5.3) & 36.9(3.1) & \textbf{29.4}(4.6) & 41.3(3.9) & 38.3(3.4)\\
6.5 & 20.9(6.6) & \textbf{15.8}(5.5) & 18.4(6.6) & 19.2(7.3) & 46.1(3.2) & \textbf{33.7}(5.0) & 42.1(4.8) & 47.9(3.7)\\
7.5 & 19.3(5.8) & \textbf{16.2}(4.8) & 18.8(7.8) & 19.6(8.2) & 54.6(12.5) & \textbf{40.7}(6.6) & 65.6(5.9) & 60.8(5.2)\\
8.5 & 20.8(6.6) & \textbf{18.8}(6.6) & 20.8(8.5) & 21.4(8.8) & 53.8(10.5) & \textbf{43.6}(7.2) & 80.9(6.3) & 73.5(4.8)\\
9.5 & 20.4(6.6) & \textbf{19.1}(6.5) & 19.6(7.5) & 20.1(8.0) & 51.7(11.3) & \textbf{48.3}(9.6) & 97.7(9.1) & 88.6(6.2)\\
\hline
\end{tabular}
\end{center}
\end{table*}

\subsection{Misspecification of $U$}
\label{ssec:simmisspec}
For our investigation of the impact of misspecifications of $U$, the baseline parameters are set to be $d=20,n_S=1000,n_T=100,\sigma_S^2=\sigma_T^2 = 1$. We randomly generate $\mathbf{X}$ where each row independently follows the standard multivariate Gaussian, and also independently generate $\mathbf{W}$ in a similar fashion. We consider $\beta_T$, the target parameter after the coordinate transformation~\eqref{eq:coordinatetransform}, as the vector of all ones. Note that specifying $\theta_T$ and $\beta_T$ is equivalent given the design matrices $\mathbf{X}$ and $\mathbf{W}$. Experiments are then performed with ground-truth values of $U \in \{0.5,1.5,2.5\}$ where, for each value of $U$, we generate $\beta_S$ to be the Nash equilibrium in problem~\eqref{eq:ubfdproblem}. We then repeat $1000$ independent simulation runs where, in each run, the design matrices $\mathbf{X},\mathbf{W}$ are kept unchanged and fresh copies of the response vectors $\mathbf{Y},\mathbf{V}$ are resampled following the
LRM~\eqref{eq:sourcemodelaftercoordinatetrans}
and~\eqref{eq:targetmodelaftercoordinatetrans}.

In calculating the optimal interpolation scheme, we assume $\sigma_S^2,\sigma_T^2$ to be known and experiment with $U_{\textrm{guess}}$ values that are equispaced within the interval $[0,3]$ as (mis)specifications of $U$. Then, we compute the average of the estimation error for $\theta_T$ in~\eqref{eq:lossfunction} over the $1000$ runs and plot the corresponding trends with respect to $U_{\textrm{guess}}$. The results are summarized in the left plot of Figure~\ref{fig:misbaseandlargeu}, where the horizontal lines are the averaged error of the basic pooling method \eqref{eq:pooling-estimator} included for comparison.
We observe that the performance gap between the proposed method (in solid lines) and the basic pooling method (in dashed lines) increases with the ground-truth value of $U$. While this gap is highest when the proposed method uses the correct $U$ value, it is robust to misspecification in $U$. 

We now vary the parameter values and observe that similar phenomena exist across the different settings:

\mbox{$\bullet$  }The ground-truth value of $U$ is changed to $\{5,10,15\}$, and we experiment with $U_{\textrm{guess}}$ values equispaced within $[0,20]$. The results are summarized in the right plot of Figure~\ref{fig:misbaseandlargeu}. 
Our proposed method works much better than the basic pooling method across the wider range of $U$ values around the ground-truth. 


\mbox{$\bullet$  }The rows of $\mathbf{X}$ are generated independently by a zero-mean Gaussian with a Toeplitz covariance matrix~\citep[Section~5.2]{ref:LiCaiLi2020},
or both $\mathbf{X}$ and $\mathbf{W}$ are generated in this way;
refer to the supplement and Figure~\ref{fig:mistoeplitz}.
The introduction of correlation 
does not impact the performance of either method when compared to the uncorrelated case in the left plot of Figure~\ref{fig:misbaseandlargeu}.


\mbox{$\bullet$  }The noise variances are changed to $\sigma_S^2=1,\sigma_T^2=5$, or $\sigma_S^2=5,\sigma_T^2 =1$; refer to the supplement and Figure~\ref{fig:misunequalvar}. The proposed method handles large variances in the source data much better than the basic pooling method, while high variance in (smaller sized) target data has no significant impact on either method.


\mbox{$\bullet$  }The dimension $d$ is changed to $5$, or $d=100$ but with $\beta_T$ three-sparse
(specifically, 
the first three elements of $\beta_T$ are one and the rest are zero);
refer to the supplement and Figure~\ref{fig:misdimchange}.
Lower dimensionality seems to improve the performance gap between the two methods as compared to the left plot of Figure~\ref{fig:misbaseandlargeu}, while this gap shortens in high-dimensions with extreme sparsity.

\begin{table*}[h]
\caption{Results Comparing the Proposed Method to Other Competing Methods on Uber$\&$Lyft Data.} \label{tab:comparisonrealdata}
\begin{center}
\begin{tabular}{|c|c|c|c|c|c|c|}
\hline
$n_S$&$n_T$&$\widehat U$&Basic&Proposed&Two-Step&Trans-Lasso \\
\hline
1000 & 100 & 19.79 & 34.54(15.10) (pooling) & \textbf{28.50}(16.53) & 1.16(2.03)$\times 10^5$ & 93.80(57.22)  \\
10000 & 1000 & 19.20 & 38.52(14.72) (target) & \textbf{36.30}(14.32) & 1.62(2.35)$\times10^5$ & 778.47(538.76) \\
1000 & 1000 & 15.05 & 37.17(10.66) (target) & \textbf{36.04}(11.08) & 8.84(13.23)$\times10^5$ & 353.49(347.70) \\
100 & 10000 & 47.97 & 30.72(9.35) (target) & \textbf{30.50}(9.35) & 1.08(15.94)$\times10^6$  & 47.59(44.61) \\
10000 & 100 & 9.29 & 29.86(6.37) (pooling) & \textbf{26.47}(8.10) & 1.35(2.46)$\times10^4$ & 101.31(67.10) \\
\hline
\end{tabular}
\end{center}
\end{table*}

\subsection{Comparisons with Competing Methods}
\label{ssec:simcompare}
We have seen from Figures~\ref{fig:misbaseandlargeu}~--~\ref{fig:misdimchange} that our method admits a broad tolerance to misspecifications of the value of $U$ relative to its true value, especially if the true value is moderately large. We next develop heuristic procedures for estimating $U$, alongside with $\sigma_S^2$ and $\sigma_T^2$, from the datasets as follows. 

\mbox{$\bullet$  }We use the usual least squares MLE estimate
    \[
    \widehat{\sigma_S^2} = \frac{1}{n_S}\sum_{i=1}^{n_S} \left(y_i - x_i^\top\widehat{\theta}_S\right)^2 ,
    \]
    with $\widehat{\theta}_S$ the ordinary least squares estimate of $\theta_S$ in~\eqref{eq:sourcelinearregmodel}.
    
\mbox{$\bullet$  }For moderate dimension and $\theta_T$ not sparse, we use a similar least squares estimate $\widehat{\sigma_T^2}$. 
    However, in high-dimensional settings, it has been observed that a more accurate estimator is given by~\citep{ref:reid2016study}
    \[
     \widehat{\sigma_T^2} = \frac{1}{n_T-\hat{s}_{\hat\gamma}}\sum_{i=1}^{n_T} \left(v_i - w_i^\top\widehat{\theta}_{T,\hat{\gamma}}\right)^2 ,
    \]
    where $\widehat{\theta}_{T,\hat\gamma}$ is the Lasso estimator with cross-validated penalization parameter $\hat\gamma$, and $\hat{s}_{\hat\gamma}$ is the number of non-zero elements in $\widehat{\theta}_{T,\hat\gamma}$. 
    
\mbox{$\bullet$  }We use a $5$-fold cross-validation
(CV)
    procedure to determine an estimate $\widehat{U}$, where the
CV
    objective is the mean-squared test error on the hold-out set. 

The experimental results in Section~\ref{ssec:simmisspec} demonstrate forms of robustness with respect to the misspecification of $U$ in our approach.
Beyond the above $5$-fold CV approach to estimate $U$ from the datasets, which delivers promising results below in comparison with state-of-the-art methods, we can use subsampling methods as an alternative to estimate $U$ whenever the source and target samples are not too scarce. 
We plan to address this issue in more detail as part of future work.

Now we compare the results from our full-fledged method with those from the basic methods discussed in Section~\ref{ssec:compare} and two recent state-of-the-art transfer learning methods in the literature, namely the two-step joint estimator proposed by~\cite{ref:bastani2021predicting} and its extension to Trans-Lasso by~\cite{ref:LiCaiLi2020}.
For the latter case, since the setting is a single source domain from which learning is transferred to a target domain, we only include for comparison the Oracle Trans-Lasso algorithm in~\cite{ref:LiCaiLi2020} (i.e., their Algorithm 1).

\subsubsection{Comparisons in Moderate-Dimensions}
\label{sec:moderatecomparison}
The parameters for the case of moderate dimensions are set to be $d=20,n_S=1000,n_T=100,\sigma_S^2=\sigma_T^2=1$. We randomly generate $\mathbf{X}$ where each row independently follows the standard multivariate Gaussian, and independently generate $\mathbf{W}$ in a similar fashion. We consider $\theta_T$ to be the vector of all ones.
Experiments are performed with ground-truth values of $U$ in $\{0.5,1.5,\ldots,9.5\}$ where, for each value of $U$, we generate $\theta_S$ to be the Nash equilibrium in problem~\eqref{eq:ubfdproblem}. We then repeat $1000$ independent simulation runs where, in each run, the design matrices $\mathbf{X},\mathbf{W}$ are kept unchanged and fresh copies of the response vectors $\mathbf{Y},\mathbf{V}$ are resampled following the
LRM~\eqref{eq:sourcelinearregmodel}
and~\eqref{eq:targetlinearregmodel}. For the methods under consideration, we report the average estimation error of $\theta_T$ in~\eqref{eq:lossfunction}, and its standard deviation, over the $1000$ runs. The results are summarized in the left-half of Table~\ref{tab:comparisonboth}.
For small $U$, the Trans-Lasso method produces the best results.
We note that these minimax optimality results of~\cite{ref:LiCaiLi2020} are established under different
assumptions, and that their results do not contradict our minimax optimality results.
For moderate to larger $U$, our proposed method attains better performance on average. We also provide additional experiments with $n_S=200,n_T=100$ in the supplement,
which exhibit consistent behaviors.


\subsubsection{Comparisons in High-Dimensions}
\label{ssec:highcomparison}
With all other parameters remaining the same, we now consider a more challenging high-dimensional setting where $d=100$. Moreover, we set $\theta_T$ to be a sparse vector where the first $20$ elements are one, and the remaining $80$ elements are zero. To deal with this high-dimensional setup, we make a simple heuristic modification to the proposed interpolation scheme~\eqref{eq:interpolationscheme} by replacing the least squares estimator $\widehat{\beta}_T$ with $\widehat{\beta}_{T,\hat\gamma}$, i.e., the Lasso estimator $\widehat{\theta}_{T,\hat\gamma}$ after the coordinate transformation~\eqref{eq:coordinatetransform}. We repeat $1000$ independent simulation runs, and report the average estimation error of $\theta_T$ in~\eqref{eq:lossfunction} and its standard deviation for the methods under consideration. The results are summarized in the right-half of Table~\ref{tab:comparisonboth}.
Our proposed method outperforms the other two methods for all $U$ values considered, somewhat surprisingly even for small $U$ since the competing methods were designed to exploit sparsity.
The supplement provides additional experiments with $n_S=200,n_T=100$
that exhibit consistent behaviors.

\subsubsection{Comparisons on Real-World Dataset}
\label{ssec:realdatacomparison}
Lastly, we compare the results from the different methods using the Uber$\&$Lyft dataset\footnote{https://www.kaggle.com/brllrb/uber-and-lyft-dataset-boston-ma} of Uber and Lyft cab
rides
collected in Boston, MA. The learning problem comprises prediction of the price using $d=32$ numerical features, including hour-of-the-day, distance, weather, and demand factors. We consider UberX as the source model and standard Lyft service as the target. The entire dataset consists of 55094 observations for the source and 51235 observations for the target, from which we compute the ground truth regression parameters; see the supplement. Since we wish to study the benefit of transfer learning, we restrict ourselves to small random subsamples. We repeat $100$ independent experiments and summarize the results in Table~\ref{tab:comparisonrealdata}. Our proposed method attains a better performance on average, by a small margin relative to the basic methods and by a large margin relative to the two-step estimator and Trans-Lasso. Notice that for this problem 
$l_q$ sparsity, $q\in[0,1]$,
(required by the last two methods) does not reasonably capture the contrast between the source and target models, due to the moderate dimensions and the existence of one dominating feature; see Table~\ref{tab:comparisonrealdatadifferentnorms} and Figure~\ref{fig:targetvssourceregestimation} in the supplement.

\clearpage
\subsubsection*{Acknowledgements}
J. Blanchet gratefully acknowledges support from the NSF via grant DMS-EPSRC 2118199 and AFOSR, as well as NSF-DMS 1915967 and AFOSR MURI FA9550-20-1-0397.

Part of this work was done while X.\ Zhang was at the IBM Thomas J.\ Watson Research Center.

\bibliography{ref}
\bibliographystyle{plainnat}


\clearpage
\appendix

\thispagestyle{empty}

\onecolumn \makesupplementtitle

%
In support of the main body of the paper, this supplement contains additional results, technical details, and complete proofs of our theoretical results.
We start by presenting proofs and related results for Theorem~\ref{prop:UB:LRM} and Theorem~\ref{prop:LB:LRM} in Sections~\ref{apdx:prop:UB:LRM} and \ref{apdx:prop:LB:LRM}, respectively.
We then present the proofs of Remark~\ref{rm:informtheory} and Proposition~\ref{prop:basicapproches} in Sections~\ref{apdx:rm:informtheory} and \ref{apdx:prop:basicapproches}, respectively.
Next, we present the proofs of Theorem~\ref{thm:glmlb}, Remark~\ref{rm:leeandcourtadelb} and Corollary~\ref{cor:multiplelr} in Sections~\ref{apdx:thm:glmlb}, \ref{apdx:rm:leeandcourtadelb} and \ref{apdx:cor:multiplelr}, respectively.
Each of these sections includes statements of the theoretical results from the main body of the paper in an effort to make the supplement self-contained. We also provide auxiliary results on the GLM upper bound in Section~\ref{apx:glmupperboundproof} and on the comparison with~\cite{ref:Kalan2020} in Section~\ref{apx:kalancomparison}.
Finally, in
Section~\ref{apdx:results},
we present an additional set of simulation results together with additional details and results for the application of a real-world dataset that complement those in the main body of the paper.

\section{Proofs and Related Results}

\subsection{Proof of Theorem~\ref{prop:UB:LRM}}
\label{apdx:prop:UB:LRM}
\textbf{Theorem~\ref{prop:UB:LRM}.}
\textit{
Under Assumption~\ref{a:pdw}, an upper bound $B$ is given by 
\begin{align}
\inf_{t_1,\ldots,t_d}\sup_{D(\theta_S,\theta_T)\leq U^2} &\mathbb{E}_{P_S,P_T}[\ell(\widehat{\theta}_{t_1,\ldots,t_d},\theta_T)]\tag{\ref{eq:ubfdproblem}}\\
&\qquad =   \sum_{i=1}^d \frac{1}{\frac{1}{\sigma_T^2} + \frac{1}{\alpha_i^\star U^2+  \frac{\sigma_S^2}{\lambda_i}}},\tag{\ref{eq:ubfdformula}}
\end{align}
where 
\[
\alpha_i^\star = \begin{cases}
 \sum_{j=i}^{K^\star}\kappa_j + \frac{1}{K^\star+1} (1-\sum_{j=1}^{K^\star}j\kappa_j) & \textrm{if } i\leq K^\star+1,\\
  0 & \textrm{if } i > K^\star + 1,
\end{cases}
\]
and 
\[
K^\star =\max_{\sum_{i=1}^K i\kappa_i\leq1,0\leq K\leq d-1}K,
\]
with
\[
\kappa_i = \frac{\sigma_S^2}{U^2}(\frac{1}{\lambda_{i+1}}-\frac{1}{\lambda_i}), \qquad i= 1,\ldots,d-1 .
\]
Moreover, the optimal estimator $\widehat{\theta}_{t_1^\star,\ldots,t_d^\star}$ satisfies 
\begin{equation}\tag{\ref{eq:optinterpscheme}}
t_i^\star = \frac{\sigma_T^2}{\sigma_T^2 + \alpha_i^\star U^2+  \frac{\sigma_S^2}{\lambda_i}}.
\end{equation}
}
%
\begin{proof}
It is more convenient to work with the reparametrization~\eqref{eq:reparameterization}. Note that
\begin{align*}
\widehat{\beta}_S &\; = \; (E^\top\mathbf{X}^\top\mathbf{X} E)^{-1} E^\top\mathbf{X}^\top \mathbf{Y} \; \sim \; \mathcal{N}(\beta_S,\sigma_S^2 \mathrm{diag}(\lambda_1^{-1},\ldots,\lambda_d^{-1})), \\
\widehat{\beta}_T &\; = \; (E^\top\mathbf{W}^\top\mathbf{W} E)^{-1} E^\top\mathbf{W}^\top \mathbf{V} \; \sim \; \mathcal{N}(\beta_T,\sigma_T^2 I).
\end{align*}
We therefore obtain
\begin{align*}
\inf_{t_1,\ldots,t_d} \sup_{\tilde D(\beta_S,\beta_T)\leq U^2} & \mathbb{E}_{P_S,P_T}[\|\widehat{\beta}-\beta_T\|_2^2] \\
  & =  \inf_{t_1,\ldots,t_d}\sup_{\|\beta_S-\beta_T\|_2^2\leq U^2} \mathrm{Tr}\left(\mathrm{diag}(t_1^2,\ldots,t_d^2)\sigma_S^2\mathrm{diag}(\lambda_1^{-1},\ldots,\lambda_d^{-1})\right) + \sum_{i=1}^d t_i^2 ((\beta_S)_i-(\beta_T)_i)^2 \\
   & \qquad\qquad\qquad + \mathrm{Tr}\left(\mathrm{diag}((1-t_1)^2,\ldots, (1-t_d)^2)\sigma_T^2\right)\\
   & = \inf_{t_1,\ldots,t_d}\sup_{\alpha_i\geq 0,\sum_{i=1}^d\alpha_i=1}\sum_{i=1}^d t_i^2\left(\frac{\sigma_S^2}{\lambda_i} + \alpha_i U^2\right) + (1-t_i)^2 \sigma_T^2\\
   & = \sup_{\alpha_i\geq 0,\sum_{i=1}^d\alpha_i=1}\inf_{t_1,\ldots,t_d}\sum_{i=1}^d t_i^2\left(\frac{\sigma_S^2}{\lambda_i} + \alpha_i U^2\right) + (1-t_i)^2 \sigma_T^2\\
   & = \sup_{\alpha_i\geq 0,\sum_{i=1}^d\alpha_i=1} \sum_{i=1}^d \frac{1}{\frac{1}{\sigma_T^2} + \frac{1}{\alpha_iU^2 + \frac{\sigma_S^2}{\lambda_i}}}\\
   & = d \sigma_T^2 - \sigma_T^4\left(\inf_{\alpha_i\geq 0,\sum_{i=1}^d\alpha_i=1} \sum_{i=1}^d\frac{1}{\alpha_i U^2 + \sigma_T^2 +\frac{\sigma_S^2}{\lambda_i}}\right),
\end{align*}
where Sion's minimax theorem~\citep{ref:sion1958minimax} is employed to swap the supremum and infinum. Further note that
\[
\sigma_T^2+\frac{\sigma_S^2}{\lambda_1}\leq\cdots\leq \sigma_T^2+\frac{\sigma_S^2}{\lambda_d}.
\]
Let
\[
\kappa_i = \frac{\sigma_S^2}{U^2}(\frac{1}{\lambda_{i+1}}-\frac{1}{\lambda_i}), \qquad i= 1,\ldots,d-1,
\]
and
\[
K^\star =\max_{\sum_{i=1}^K i\kappa_i\leq1,0\leq K\leq d-1}K.
\]
It is then easy to see that the solution of 
\[
\inf_{\alpha_i\geq 0,\sum_{i=1}^d\alpha_i=1} \sum_{i=1}^d\frac{1}{\alpha_i U^2 + \sigma_T^2 +\frac{\sigma_S^2}{\lambda_i}}
\]
is given by
\[
\alpha_i^\star = \begin{cases}
 \sum_{j=i}^{K^\star}\kappa_j + \frac{1}{K^\star+1} (1-\sum_{j=1}^{K^\star}j\kappa_j) & \textrm{ if }i\leq K^\star+1,\\
  0 & \textrm{ if }i > K^\star + 1.
\end{cases}
\]
Hence, we have
\[
\inf_{t_1,\ldots,t_d}\sup_{\tilde D(\beta_S,\beta_T)\leq U^2}\mathbb{E}_{P_S,P_T}[\|\widehat{\beta}-\beta_T\|_2^2]=  \sum_{i=1}^d \frac{1}{\frac{1}{\sigma_T^2} + \frac{1}{\alpha_i^\star U^2+  \frac{\sigma_S^2}{\lambda_i}}},
\]
and
\[
t_i^\star = \frac{\frac{1}{\alpha_i^\star U^2 + \frac{\sigma_S^2}{\lambda_i}}}{\frac{1}{\sigma_T^2} + \frac{1}{\alpha_i^\star U^2 + \frac{\sigma_S^2}{\lambda_i}}} = \frac{\sigma_T^2}{\sigma_T^2 + \alpha_i^\star U^2+  \frac{\sigma_S^2}{\lambda_i}},\qquad i=1,\ldots,d,
\]
thus completing the proof.
\end{proof}

\subsection{Proof of Theorem~\ref{prop:LB:LRM}}
\label{apdx:prop:LB:LRM}
In order to obtain our lower bound results, we shall make use of the following lemma which concerns the admissibility of the hypotheses in applying Le Cam's or Fano's method~\citep[Chapter~2]{ref:tsybakov2008introduction}. 
\begin{lemma}\label{lemma:onedimconditions}
For any $\gamma_i>0$, $U_i>0,i=0,\ldots,M,$ and $K\geq \frac{1}{4}$, there exists $\beta_i^{(j)}\in\mathbb{R},i=0,\ldots,M,j=0,1,$ such that the following conditions hold simultaneously:
\begin{enumerate}
    \item $\left(\beta_0^{(0)} - \beta_0^{(1)}\right)^2 = \frac{1}{K}\frac{1}{\gamma_0+\sum_{k=1}^M\frac{1}{U_k^2 + \frac{1}{\gamma_k}}};$
    \item $\left|\beta_0^{(j)}-\beta_i^{(j)}\right|\leq U_i, \qquad i=1,\ldots,M,j=0,1;$
    \item $\left(\beta_i^{(0)}-\beta_i^{(1)}\right)^2\leq \frac{1}{K}\frac{\frac{1}{U_i^2+\frac{1}{\gamma_i}}\frac{1}{\gamma_i}}{\gamma_0+\sum_{k=1}^M\frac{1}{U^2_k+\frac{1}{\gamma_k}}}, \qquad i=1,\ldots,M$.
\end{enumerate}
Under these conditions, we have
\[
\sum_{i=0}^M \gamma_i\left(\beta_i^{(0)}-\beta_i^{(1)}\right)^2\leq \frac{1}{K}.
\]
\end{lemma}
\begin{proof}
For the three conditions to hold simultaneously, it suffices to show (using the triangle inequality twice to bound condition $2$ in terms of conditions $1$ and $3$) that, for $i=1,\ldots,M$,
\[
\sqrt{\frac{1}{\gamma_0+\sum_{k=1}^M\frac{1}{U_k^2 + \frac{1}{\gamma_k}}}} - \sqrt{\frac{\frac{1}{U_i^2+\frac{1}{\gamma_i}}\frac{1}{\gamma_i}}{\gamma_0+\sum_{k=1}^M\frac{1}{U^2_k+\frac{1}{\gamma_k}}}}\leq 2\sqrt{K} U_i.
\]
This problem is equivalent to showing that
\[
1 - \sqrt{\frac{1}{\gamma_i}\frac{1}{U_i^2 +\frac{1}{\gamma_i}}}\leq 2\sqrt{K} U_i \sqrt{\gamma_0+\sum_{k=1}^M\frac{1}{U_k^2+\frac{1}{\gamma_k}}}.
\]
In fact, we can show a tighter bound
\[
1 - \sqrt{\frac{1}{\gamma_i}\frac{1}{U_i^2 +\frac{1}{\gamma_i}}}\leq 2\sqrt{K} U_i \sqrt{\gamma_0+\frac{1}{U_i^2+\frac{1}{\gamma_i}}},
\]
which can be rewritten as
\[
\sqrt{U_i^2 +\frac{1}{\gamma_i}}-\sqrt{\frac{1}{\gamma_i}}\leq 2\sqrt{K}U_i\sqrt{1+\gamma_0(U_i^2 +\frac{1}{\gamma_i})},
\]
or further rewritten as
\begin{equation}\label{eq:lemmasimple}
\frac{U_i^2}{\sqrt{U_i^2+\frac{1}{\gamma_i}}+\sqrt{\frac{1}{\gamma_i}}}\leq 2\sqrt{K}U_i\sqrt{1+\gamma_0(U_i^2 +\frac{1}{\gamma_i})}.
\end{equation}
Since $K\geq\frac{1}{4}$ and 
\[
\frac{U_i}{\sqrt{U_i^2+\frac{1}{\gamma_i}}+\sqrt{\frac{1}{\gamma_i}}}\leq1,
\]
we conclude that the desired result~\eqref{eq:lemmasimple} holds. The last part of the claim in Lemma~\ref{lemma:onedimconditions} follows easily.
\end{proof}

We now are in a position to prove Theorem~\ref{prop:LB:LRM} via Le Cam's method~\citep[Chapter~2]{ref:tsybakov2008introduction}, which we restate as follows.

\textbf{Theorem~\ref{prop:LB:LRM}.}
\textit{
Under Assumption~\ref{a:pdw}, a lower bound $L$ is given by 
\begin{equation}
\inf_{\widehat{\theta}}\sup_{D(\theta_S,\theta_T)\leq U^2}\mathbb{E}_{P_S,P_T}[\ell(\widehat\theta,\theta_T)]
\geq\frac{\exp{\left(-\frac{1}{2}\right)}}{16}\sum_{i=1}^d \frac{1}{\frac{1}{\sigma_T^2} + \frac{1}{\alpha_i^\star U^2+  \frac{\sigma_S^2}{\lambda_i}}}.\tag{\ref{eq:lbformulalinearregression}}
\end{equation}
}
%
\begin{proof}
It is more convenient to work with the reparametrization~\eqref{eq:reparameterization}. Note that
\begin{align*}
    \inf_{\widehat{\beta}}\sup_{\tilde D(\beta_S,\beta_T)\leq U^2}\mathbb{E}_{P_S,P_T}[\tilde\ell(\widehat\beta,\beta_T)]& \geq \inf_{\widehat{\beta}}\sup_{((\beta_S)_i-(\beta_T)_i)^2\leq \alpha_i^\star U^2\forall i}\mathbb{E}_{P_S,P_T}[\tilde\ell(\widehat\beta,\beta_T)]\\
    &\geq \sum_{i=1}^d \inf_{\widehat{\beta}_i}\sup_{((\beta_S)_i-(\beta_T)_i)^2\leq \alpha_i^\star U^2}\mathbb{E}_{P_S,P_T}[(\widehat\beta_i-(\beta_T)_i)^2].
\end{align*}
Consider the singular value decomposition of $\mathbf{W}E$, for which we obtain
\[
\mathbf{W} E = PD Q = P \begin{pmatrix} Q\\ 0 \end{pmatrix},
\]
since $E^\top\mathbf{W}^\top\mathbf{W}E = I$ and where $P,Q$ are orthogonal matrices of appropriate dimensions. We therefore have
\[
\begin{pmatrix} E^\top \mathbf{W}^\top \\ (0\,\, I) P^\top\end{pmatrix} \mathbf{W} E = \begin{pmatrix} Q^\top & 0\\
0 & I\end{pmatrix} P^\top  \mathbf{W} E = \begin{pmatrix} I \\0\end{pmatrix},
\]
and thus
\begin{equation}\label{eq:tildev}
\tilde{\mathbf{V}} = \begin{pmatrix} E^\top \mathbf{W}^\top \\ (0\,\, I) P^\top\end{pmatrix}\mathbf{V} =  \begin{pmatrix} I \\0\end{pmatrix}\beta_T + \tilde\eta,\quad \tilde\eta\sim\mathcal{N}\left(0,\sigma_T^2\begin{pmatrix} I & 0 \\ 0 & \Gamma\end{pmatrix}\right),
\end{equation}
where $\Gamma$ is some positive-definite matrix that does not concern us in the following calculations. Similarly, considering the singular value decomposition of $\mathbf{X}E$, we obtain
\[
\mathbf{X}E = \tilde P \tilde D \tilde Q =  \tilde P \begin{pmatrix} \mathrm{diag}(\lambda_1^{1/2},\ldots,\lambda_d^{1/2})\tilde Q\\ 0 \end{pmatrix},
\]
since $E^\top\mathbf{X}^\top\mathbf{X}E =\mathrm{diag}(\lambda_1,\ldots,\lambda_d)$ and
where $\tilde P,\tilde Q$ are orthogonal matrices of appropriate dimensions.
We therefore have
\begin{equation}\label{eq:tildex}
\tilde{\mathbf{Y}} =  \begin{pmatrix} E^\top \mathbf{X}^\top \\ (0\,\, I) \tilde P^\top\end{pmatrix}\mathbf{Y} =  \begin{pmatrix} \mathrm{diag}(\lambda_1,\ldots,\lambda_d) \\0\end{pmatrix}\beta_S + \tilde\epsilon,\quad \tilde\epsilon\sim\mathcal{N}\left(0,\sigma_S^2\begin{pmatrix} \mathrm{diag}(\lambda_1,\ldots,\lambda_d) & 0 \\ 0 & \tilde\Gamma\end{pmatrix}\right),
\end{equation}
where $\tilde\Gamma$ is some positive-definite matrix that does not concern us in the following calculations. By Le Cam's method~\citep[Chapter~2]{ref:tsybakov2008introduction}, we then obtain
\[
\inf_{\widehat{\beta}_i}\sup_{((\beta_S)_i-(\beta_T)_i)^2\leq \alpha_i^\star U^2}\mathbb{E}_{P_S,P_T}[(\widehat\beta_i-(\beta_T)_i)^2]\geq \frac{((\beta_T)_i^0 - (\beta_T)_i^1)^2}{16} \exp{\left(- \frac{((\beta_T)_i^0 - (\beta_T)_i^1)^2}{2\sigma_T^2} - \frac{\lambda_i((\beta_S)_i^0 - (\beta_S)_i^1)^2}{2\sigma_S^2}\right)},
\]
for any
$$|(\beta_T)_i^0-(\beta_S)_i^0|=|(\beta_T^0- \beta_{S}^0)_i|\leq \sqrt{\alpha_i^\star} U$$ and  $$|(\beta_T)_i^1-(\beta_S)_i^1|=|(\beta_T^1- \beta_{S}^1)_i|\leq \sqrt{\alpha_i^\star} U.$$
From Lemma~\ref{lemma:onedimconditions}, upon choosing $K=1$, we know there exists  $(\beta_T)_i^0,(\beta_T)_i^1,(\beta_S)_i^0,(\beta_S)_i^1$ such that 
\[
(\beta_T^0- \beta_{S}^0)_i^2\leq \alpha_i^\star U^2, (\beta_T^1- \beta_{S}^1)_i^2\leq \alpha_i^\star U^2, (\beta_T^0-\beta_T^1)_i^2 = \frac{1}{\frac{1}{\alpha_i^\star U^2 +\frac{\sigma^2_{S}}{\lambda_i}}+\frac{1}{\sigma_T^2}}
\]
and
\[
(\beta_{S}^0-\beta_{S}^1)_i^2\leq \frac{\frac{1}{\alpha_i^\star U^2 +\frac{\sigma_{S}^2}{\lambda_i}}\frac{\sigma_{S}^2}{\lambda_i}}{\frac{1}{\alpha_i^\star U^2 +\frac{\sigma_{S}^2}{\lambda_i}}+\frac{1}{\sigma_T^2}}.
\]
We therefore have
\begin{align*}
    \frac{((\beta_T)_i^0 - (\beta_T)_i^1)^2}{16} \exp{\left(- \frac{((\beta_T)_i^0 - (\beta_T)_i^1)^2}{2\sigma_T^2} - \frac{\lambda_i((\beta_S)_i^0 - (\beta_S)_i^1)^2}{2\sigma_S^2}\right)}\geq \frac{\exp\left(-\frac{1}{2}\right)}{16} \frac{1}{\frac{1}{\alpha_i^\star U^2 +\frac{\sigma^2_{S}}{\lambda_i}}+\frac{1}{\sigma_T^2}},
\end{align*}
and thus conclude the lower bound
\[
\inf_{\widehat{\beta}}\sup_{\tilde D(\beta_S,\beta_T)\leq U^2}\mathbb{E}_{P_S,P_T}[\tilde\ell(\widehat\beta,\beta_T)]\geq\frac{\exp\left(-\frac{1}{2}\right)}{16}\sum_{i=1}^d \frac{1}{\frac{1}{\alpha_i^\star U^2 +\frac{\sigma^2_{S}}{\lambda_i}}+\frac{1}{\sigma_T^2}}.
\]
\end{proof}

\subsection{Proof of Remark~\ref{rm:informtheory}}
\label{apdx:rm:informtheory}
\textbf{Remark~\ref{rm:informtheory}.}
\textit{
Using the channel capacity of a non-Gaussian additive noise channel~\citep{ref:IHARA1978capacity}, we can improve the uniform constant $\displaystyle \; \frac{\exp{\left(-\frac{1}{2}\right)}}{16} \;$ to
\[
 \max\left\{\frac{\exp{\left(-\frac{1}{2}\right)}}{16},\left(\frac{\frac{\sigma_S^2}{\lambda_i}}{\alpha_i^\star U^2+\frac{\sigma_S^2}{\lambda_i}}\right)^2\right\}
 \]
for the $i$-th summand in~\eqref{eq:lbformulalinearregression}. Note that the second term is $1$ if $\alpha_i^\star=0$, and it is arbitrarily close to $1$ if $U$ is sufficiently small.
}
%
\begin{proof}
It is well known that the minimax risk is lower bounded by the Bayesian risk
\[
\inf_{\widehat{\beta}}\sup_{\tilde D(\beta_S,\beta_T)\leq U^2}\mathbb{E}_{P_S,P_T}[\tilde\ell(\widehat\beta,\beta_T)]\geq \inf_{\widehat{\beta}}\mathbb{E}[\|\widehat\beta-\beta_T\|_2^2],
\]
where the expectation on the right-hand side refers to a fixed design model (i.e., the predictors in the source and target are given), there is a prior on both $\beta_S$ and $\beta_T$, and the responses, conditional on the prior, follow the $P_S$ and $P_T$ models for source and target environments, respectively. We assume independent priors $(\beta_T)_i\sim\mathcal{N}(0,\sigma^2)$, and further assume $(\beta_S)_i = (\beta_T)_i +\sqrt{\alpha_i^\star}\Delta_i$ where $\Delta_i$ assigns a probability of $0.5$ to $-U$ and a probability of $0.5$ to $U$.
By the maximum entropy of the Gaussian distribution and the data processing inequality, we have
\[
\inf_{\widehat{\beta}}\mathbb{E}[\|\widehat\beta-\beta_T\|_2^2]\geq\frac{1}{2\pi e}\sum_{i=1}^d e^{2h((\beta_T)_i)-2I(\mathbf{Y},\mathbf{V};(\beta_T)_i)}.
\]
Since mutual information is invariant under invertible transformations, we obtain
\[
I(\mathbf{Y},\mathbf{V};(\beta_T)_i) = I(\tilde{\mathbf{V}},\tilde{\mathbf{Y}};(\beta_T)_i) = I(\tilde{\mathbf{V}}_i,\tilde{\mathbf{Y}}_i;(\beta_T)_i),
\]
where $\tilde V$ and $\tilde Y$ are invertible transformations of $V$ and $Y$; refer to~\eqref{eq:tildev} and~\eqref{eq:tildex}.
We also know that
\[
\tilde{\mathbf{V}}_i = (\beta_T)_i + \tilde\eta_i
\]
and
\[
\frac{\tilde{\mathbf{Y}}_i}{\lambda_i} = (\beta_T)_i +\sqrt{\alpha_i^\star}\Delta_i + \frac{1}{\lambda_i}\tilde\epsilon_i.
\]
Noting the decomposition
\[
I(\tilde{\mathbf{V}}_i,\tilde{\mathbf{Y}}_i;(\beta_T)_i) = I(\tilde{\mathbf{V}}_i;(\beta_T)_i) + I (\tilde{\mathbf{Y}}_i;(\beta_T)_i| \tilde{\mathbf{V}}_i),
\]
then, as $\sigma^2\to\infty$, we have
\[
I(\tilde{\mathbf{V}}_i;(\beta_T)_i) \sim \frac{1}{2}(\log(\sigma^2)- \log(\sigma_T^2)).
\]
For the second term of this decomposition, we obtain
$$I (\tilde{\mathbf{Y}}_i;(\beta_T)_i| \tilde{\mathbf{V}}_i)  = \mathbb{E}_{\tilde{\mathbf{V}}_i}[I (\tilde{\mathbf{Y}}_i;(\beta_T)_i| \tilde{\mathbf{V}}_i =\tilde v_i)].$$
Due to conditional independence, we know that $\tilde{\mathbf{Y}}_i|((\beta_T)_i, \tilde{\mathbf{V}}_i =\tilde v_i)$ has the same distribution as $\tilde{\mathbf{Y}}_i|(\beta_T)_i$. Hence, we see that
\[
I (\tilde{\mathbf{Y}}_i;(\beta_T)_i| \tilde{\mathbf{V}}_i =\tilde v_i) = I(\tilde{\mathbf{Y}}_i;(\tilde{\beta}_T)_i),
\]
where
\[
(\tilde{\beta}_T)_i\sim \mathcal{N}((1+\frac{1}{\sigma^2})^{-1}v_{i},\sigma_T^2(1 +\frac{1}{\sigma^2})^{-1}).
\]
From the non-Gaussian additive noise channel capacity~\citep{ref:IHARA1978capacity}, we have
\[
I (\tilde{\mathbf{Y}}_i;(\tilde{\beta}_T)_i)\leq \frac{1}{2}\log(1+\frac{\sigma_T^2(1+\sigma^{-2})^{-1}}{\alpha_i^\star U^2 + \frac{\sigma_S^2}{\lambda_i}})+ \mathrm{KL}(P_{\sqrt{\alpha_i^\star}\Delta_i+\frac{1}{\lambda_i}\tilde\epsilon_i}\|\mathcal{N}(0, \alpha_i^\star U^2 + \frac{\sigma_S^2}{\lambda_i})),
\]
and by the convexity of KL divergence, we obtain
\begin{align*}
& \mathrm{KL}(P_{\sqrt{\alpha_i^\star}\Delta_i+\frac{1}{\lambda_i}\tilde\epsilon_i}\|\mathcal{N}(0, \alpha_i^\star U^2 + \frac{\sigma_S^2}{\lambda_i}))\\
&\qquad \leq \frac{1}{2}\mathrm{KL}(\mathcal{N}(\sqrt{\alpha_i^\star}U,\frac{\sigma_S^2}{\lambda_i})\|\mathcal{N}(0, \alpha_i^\star U^2 + \frac{\sigma_S^2}{\lambda_i})) +\frac{1}{2}\mathrm{KL}(\mathcal{N}(-\sqrt{\alpha_i^\star}U,\frac{\sigma_S^2}{\lambda_i})\|\mathcal{N}(0, \alpha_i^\star U^2 + \frac{\sigma_S^2}{\lambda_i}))\\
&\qquad \leq \log(\alpha_i^\star U^2+\frac{\sigma_S^2}{\lambda_i})-\log(\frac{\sigma_S^2}{\lambda_i}).
\end{align*}
We therefore have
\[
\lim_{\sigma^2\to\infty} h((\beta_T)_i) - I(\tilde{\mathbf{V}}_i,\tilde{\mathbf{Y}}_i;(\beta_T)_i) \geq \frac{1}{2}\left(\log(2\pi e) -\log(\frac{1}{\sigma_T^2} +\frac{1}{\alpha^\star_iU^2+\frac{\sigma_S^2}{\lambda_i}})\right) +\log(\frac{\sigma_S^2}{\lambda_i}) - \log(\alpha_i^\star U^2+\frac{\sigma_S^2}{\lambda_i}),
\]
and thus conclude the lower bound
\[
 \inf_{\widehat{\beta}}\mathbb{E}[\|\widehat\beta-\beta_T\|_2^2]\geq\sum_{i=1}^d \left(\frac{\frac{\sigma_S^2}{\lambda_i}}{\alpha_i^\star U^2+\frac{\sigma_S^2}{\lambda_i}}\right)^2\cdot \frac{1}{\frac{1}{\sigma_T^2} + \frac{1}{\alpha_i^\star U^2 + \frac{\sigma_S^2}{\lambda_i}}}.
\]
\end{proof}

\subsection{Proof of Proposition~\ref{prop:basicapproches}}
\label{apdx:prop:basicapproches}

\textbf{Proposition~\ref{prop:basicapproches}.}
\textit{
\begin{enumerate}
\item
For the
LRM
based solely on the source dataset, the estimator $\widehat{\theta}_S=E\widehat{\beta}_S$ satisfies
    \[
        \sup_{D(\theta_S,\theta_T)\leq U^2}\mathbb{E}_{P_S,P_T}[\ell(\widehat{\theta}_S,\theta_T)]  = U^2  + \sigma_S^2 \sum_{i=1}^d \lambda_i^{-1} .
    \]
\item
For the
LRM
based solely on the target dataset, the estimator $\widehat{\theta}_T= E\widehat{\beta}_T$ satisfies
    \[
    \sup_{D(\theta_S,\theta_T)\leq U^2}\mathbb{E}_{P_S,P_T}[\ell(\widehat{\theta}_T,\theta_T)]  =  d \sigma_T^2 .
    \]
\item
Finally, for the
LRM
based on pooling the source and target datasets, the estimator 
    \begin{equation}
    \widehat{\theta}_P = \left(\begin{pmatrix}
     \mathbf{X}^\top \;\; \mathbf{W}^\top\end{pmatrix}\begin{pmatrix}
    \mathbf{X}\\\mathbf{W}\end{pmatrix}\right)^{-1} \cdot \begin{pmatrix}
     \mathbf{X}^\top \;\; \mathbf{W}^\top\end{pmatrix}\begin{pmatrix} \mathbf{Y}\\\mathbf{V}\end{pmatrix}
     \tag{\ref{eq:pooling-estimator}}
    \end{equation}
    satisfies
    \begin{equation*}
    \sup_{D(\theta_S,\theta_T)\leq U^2}\mathbb{E}_{P_S,P_T}[\ell(\widehat{\theta}_P,\theta_T)]
    =   U^2\max_{1\leq i\leq d}\left\{\left(\frac{\lambda_i}{1+\lambda_i}\right)^2\right\}
    + \sigma_T^2\sum_{i=1}^d \left(\frac{1}{1+\lambda_i}\right)^2
    + \sigma_S^2 \sum_{i=1}^d \frac{\lambda_i}{(1+\lambda_i)^2} .
    \end{equation*}
\end{enumerate}
}
%
%
\begin{proof}
It is more convenient to work with the reparametrization~\eqref{eq:reparameterization}. Note that
\begin{align*}
\widehat{\beta}_S &= (E^\top\mathbf{X}^\top\mathbf{X} E)^{-1} E^\top\mathbf{X}^\top \mathbf{Y}\; \sim\; \mathcal{N}(\beta_S,\sigma_S^2 \mathrm{diag}(\lambda_1^{-1},\ldots,\lambda_d^{-1})),\\
\widehat{\beta}_T &= (E^\top\mathbf{W}^\top\mathbf{W} E)^{-1} E^\top\mathbf{W}^\top \mathbf{V} \; \sim\; \mathcal{N}(\beta_T,\sigma_T^2 I).
\end{align*}
We then have for the estimator $\widehat{\theta}_S$
\begin{align*}
 \sup_{\tilde D(\beta_S,\beta_T)\leq U^2}\mathbb{E}_{P_S,P_T}[\|\widehat{\beta}_S-\beta_T\|_2^2] & =\sup_{\|\beta_S-\beta_T\|_2^2\leq U^2} \mathrm{Tr}\left(\sigma_S^2\cdot \mathrm{diag}(\lambda_1^{-1},\ldots,\lambda_d^{-1})\right) + \sum_{i=1}^d  ((\beta_S)_i-(\beta_T)_i)^2\\
 & =   \sigma_S^2 \sum_{i=1}^d \lambda_i^{-1} +U^2,
 \end{align*}
 and similarly for the estimator $\widehat{\theta}_T$
 \begin{align*}
 \sup_{\tilde D(\beta_S,\beta_T)\leq U^2}\mathbb{E}_{P_S,P_T}[\|\widehat{\beta}_T-\beta_T\|_2^2] & =\sup_{\|\beta_S-\beta_T\|_2^2\leq U^2} \mathrm{Tr}\left(\sigma_T^2\cdot I\right) =   \sigma_T^2 d.
 \end{align*}
 For the pooling estimator~\eqref{eq:pooling-estimator}, consider its reparametrization
  \[
    \widehat{\beta}_P = \left(\begin{pmatrix}
    E^\top \mathbf{X}^\top & E^\top \mathbf{W}^\top\end{pmatrix}\begin{pmatrix}
    \mathbf{X}E\\\mathbf{W}E\end{pmatrix}\right)^{-1}\begin{pmatrix}
    E^\top \mathbf{X}^\top & E^\top \mathbf{W}^\top\end{pmatrix}\begin{pmatrix} \mathbf{Y}\\\mathbf{V}\end{pmatrix} ,
    \]
 whose bias we can compute as
 \begin{align*}
 \mathbb{E}_{P_S,P_T}[\widehat{\beta}_P] -\beta_T & = \left(E^\top\bX^\top\bX E + E^\top\bW^\top\bW E\right)^{-1}\left(E^\top\bX^\top\bX E\beta_S + E^\top\bW^\top\bW E\beta_T\right) -\beta_T\\
 & = \mathrm{diag}\left(1+\lambda_1,\ldots,1+\lambda_d\right)^{-1} \left(\mathrm{diag}\left(\lambda_1,\ldots,\lambda_d\right)\beta_S + \beta_T\right) - \beta_T\\
 & = \mathrm{diag}\left(\lambda_1/(1+\lambda_1),\ldots,\lambda_d/(1+\lambda_d)\right)(\beta_S-\beta_T),
 \end{align*}
 and whose variance we can compute as
 \begin{align*}
   \mathbb{E}_{P_S,P_T}[\|\widehat{\beta}_P\|_2^2] & =  \mathrm{Tr}\left(  \mathrm{diag}\left(1+\lambda_1,\ldots,1+\lambda_d\right)^{-1}\left(\mathrm{diag}\left(\lambda_1,\ldots,\lambda_d\right)\sigma_S^2 +\sigma_T^2 I\right)\mathrm{diag}\left(1+\lambda_1,\ldots,1+\lambda_d\right)^{-1}\right)\\
   & = \sigma_T^2\sum_{i=1}^d \left(\frac{1}{1+\lambda_i}\right)^2
    + \sigma_S^2 \sum_{i=1}^d \frac{\lambda_i}{(1+\lambda_i)^2}.
 \end{align*}
 We therefore obtain
\begin{align*}
 \sup_{\tilde D(\beta_S,\beta_T)\leq U^2}\mathbb{E}_{P_S,P_T}[\|\widehat{\beta}_P-\beta_T\|_2^2] & =\sup_{\|\beta_S-\beta_T\|_2^2\leq U^2}\sigma_T^2\sum_{i=1}^d \left(\frac{1}{1+\lambda_i}\right)^2
    + \sigma_S^2 \sum_{i=1}^d \frac{\lambda_i}{(1+\lambda_i)^2} + \sum_{i=1}^d \left(\frac{\lambda_i}{1+\lambda_i}\right)^2(\beta_S-\beta_T)_i^2\\
 & =   U^2\max_{1\leq i\leq d}\left\{\left(\frac{\lambda_i}{1+\lambda_i}\right)^2\right\}
    + \sigma_T^2\sum_{i=1}^d \left(\frac{1}{1+\lambda_i}\right)^2
    + \sigma_S^2 \sum_{i=1}^d \frac{\lambda_i}{(1+\lambda_i)^2},
 \end{align*}
 thus completing the proof.
\end{proof}

To see that the worst-case risk~\eqref{eq:ubfdformula} is also smaller than that of the pooling method, we can compute
\begin{align*}
     \frac{1}{\frac{1}{\sigma_T^2} + \frac{1}{\alpha_i^\star U^2+  \frac{\sigma_S^2}{\lambda_i}}} - & \sigma_T^2\left(\frac{1}{1+\lambda_i}\right)^2 - \sigma_S^2 \frac{\lambda_i}{(1+\lambda_i)^2}\\
     & \qquad = \left(\frac{\lambda_i}{1+\lambda_i}\right)^2 \frac{\lambda_i\sigma_T^2 +(2\sigma_T^2-\sigma_S^2)}{\lambda_i\sigma_T^2+\lambda_i\alpha_i^\star U^2+\sigma_S^2}\alpha_i^\star U^2 + \frac{\lambda_i}{(1+\lambda_i)^2}\frac{2\sigma_S^2\sigma_T^2 -\sigma_S^4-\sigma_T^4}{\lambda_i\sigma_T^2+\lambda_i\alpha_i^\star U^2+\sigma_S^2}\\
     & \qquad = \left(\frac{\lambda_i}{1+\lambda_i}\right)^2 \frac{\lambda_i\sigma_T^2 +(2\sigma_T^2-\sigma_S^2)+ \frac{2\sigma_S^2\sigma_T^2-\sigma_S^4-\sigma_T^4}{\lambda_i\alpha_i^\star U^2}}{\lambda_i\sigma_T^2+\lambda_i\alpha_i^\star U^2+\sigma_S^2}\alpha_i^\star U^2.
\end{align*}
It is then readily verified that
\[
(2\sigma_T^2-\sigma_S^2)+ \frac{2\sigma_S^2\sigma_T^2-\sigma_S^4-\sigma_T^4}{\lambda_i\alpha_i^\star U^2}\leq \lambda_i\alpha_i^\star U^2 +\sigma_S^2,
\]
as this is equivalent to
\[
2(\sigma_T^2-\sigma_S^2)\lambda_i\alpha_i^\star U^2\leq \left(\lambda_i\alpha_i^\star U^2\right)^2 + \left(\sigma_S^2-\sigma_T^2\right)^2.
\]
Hence, we have the desired inequality
\[
\sum_{i=1}^d \frac{1}{\frac{1}{\sigma_T^2} + \frac{1}{\alpha_i^\star U^2+  \frac{\sigma_S^2}{\lambda_i}}}\leq  U^2\max_{1\leq i\leq d}\left\{\left(\frac{\lambda_i}{1+\lambda_i}\right)^2\right\}+ \sigma_T^2\sum_{i=1}^d \left(\frac{1}{1+\lambda_i}\right)^2 + \sigma_S^2 \sum_{i=1}^d \frac{\lambda_i}{(1+\lambda_i)^2}.
\]

\subsection{Proof of Theorem~\ref{thm:glmlb}}
\label{apdx:thm:glmlb}
\textbf{Theorem~\ref{thm:glmlb}.}
\textit{
A lower bound of the minimax risk corresponding to the GLMs is given by
\begin{equation*}
\inf_{\widehat{\theta}}\sup_{D(\theta_{S_m},\theta_T)\leq U_m^2, \forall m} \mathbb{E}[\ell(\widehat{\theta},\theta_T)] \geq\frac{e^{-1}}{800} \frac{d}{\sum_{m=1}^M\frac{1}{\frac{U_m^2}{d} +\frac{a^{(m)}(\sigma_{S_m})}{C_{S_m}\lambda_1^{(m)}}}+\frac{C_T}{b(\sigma_T)}} .
\end{equation*}
}
%
\begin{proof}
First consider the case where $d\leq 100$, for which we use Le Cam's method~\citep[Chapter~2]{ref:tsybakov2008introduction}. For two pairs of parameters $(\theta_{S_1}^0,\ldots,\theta_{S_M}^0,\theta_T^0)$ and $(\theta_{S_1}^1,\ldots,\theta_{S_M}^1,\theta_T^1)$,  we have
\[
\inf_{\widehat{\theta}}\sup_{D(\theta_{S_m},\theta_T)\leq U_m^2\forall m} \mathbb{E}[\ell(\widehat{\theta},\theta_T)]\geq \frac{\ell(\theta_T^0,\theta_T^1)}{16}\exp\left\{-\mathrm{KL}((\theta_{S_1}^0,\ldots,\theta_{S_M}^0,\theta_T^0) ; (\theta_{S_1}^1,\ldots,\theta_{S_M}^1,\theta_T^1))\right\}.
\]
By independence, we note that
\begin{align*}
    \mathrm{KL}((\theta_{S_1}^0,\ldots,\theta_{S_M}^0,\theta_T^0) ; (\theta_{S_1}^1,\ldots,\theta_{S_M}^1,\theta_T^1)) & = \sum_{m=1}^M \mathrm{KL}(\theta_{S_m}^0;\theta_{S_m}^1) +\mathrm{KL}(\theta_T^0;\theta_T^1)
\end{align*}
and
\begin{align*}
    \mathrm{KL}(\theta_{S_m}^0;\theta_{S_m}^1) & = \frac{1}{a^{(m)}(\sigma_{S_m})}\sum_{i=1}^{n_{S_m}}\big(\Psi^{(m)}(\langle x_i^{(m)},\theta_{S_m}^1\rangle ) - \Psi^{(m)}(\langle x_i^{(m)},\theta_{S_m}^0\rangle) \\
    &\qquad\qquad\qquad\qquad - \langle (\Psi^{(m)})^{'}(\langle x_i^{(m)},\theta_{S_m}^0\rangle)x^{(m)}_i,\theta_{S_m}^1-\theta_{S_m}^0\rangle\big)\\
    &\leq \frac{1}{a^{(m)}(\sigma_{S_m})}\sum_{i=1}^{n_{S_m}} \frac{1}{2}C_{S_m}\sum_{j,k} x_{ij}^{(m)}x_{ik}^{(m)}(\theta_{S_m}^1-\theta_{S_m}^0)_{j}(\theta_{S_m}^1-\theta_{S_m}^0)_{k}\\
    &\leq \frac{C_{S_m}}{2a^{(m)}(\sigma_{S_m})}(\theta_{S_m}^1-\theta_{S_m}^0)^\top (\mathbf{X}^{(m)})^\top\mathbf{X}^{(m)}(\theta_{S_m}^1-\theta_{S_m}^0)\\
    &\leq \frac{\lambda_1^{(m)}C_{S_m}}{2a^{(m)}(\sigma_{S_m})}(\theta_{S_m}^1-\theta_{S_m}^0)^\top \mathbf{W}^\top\mathbf{W}(\theta_{S_m}^1-\theta_{S_m}^0).
\end{align*}
Similarly, we obtain
\[
\mathrm{KL}(\theta_T^0;\theta_T^1) \leq \frac{C_T}{2b(\sigma_T)}(\theta_T^1-\theta_T^0)^\top \mathbf{W}^\top\mathbf{W}(\theta_T^1-\theta_T^0) .
\]
Then, by Le Cam's bound~\citep[Chapter~2]{ref:tsybakov2008introduction}, we have
\begin{align*}
    \inf_{\widehat{\theta}}\sup_{D(\theta_{S_m},\theta_T)\leq U_m^2\forall m}\mathbb{E}[\ell(\widehat{\theta},\theta_T)]
    & \geq \frac{\ell(\theta_T^0,\theta_T^1)}{16}\exp\left\{-\mathrm{KL}((\theta_{S_1}^0,\ldots,\theta_{S_M}^0,\theta_T^0) ; (\theta_{S_1}^1,\ldots,\theta_{S_M}^1,\theta_T^1))\right\}\\
    & \geq \frac{(\theta_T^1-\theta_T^0)^\top \mathbf{W}^\top\mathbf{W}(\theta_T^1-\theta_T^0)}{16} \exp\left\{-\frac{C_T}{2b(\sigma_T)}(\theta_T^1-\theta_T^0)^\top \mathbf{W}^\top\mathbf{W}(\theta_T^1-\theta_T^0)\right\} \; \cdot\\
    &\qquad\qquad\qquad\qquad\qquad \exp\left\{-\sum_{m=1}^M\frac{\lambda_1^{(m)}C_{S_m}}{2a^{(m)}(\sigma_{S_m})}(\theta_{S_m}^1-\theta_{S_m}^0)^\top \mathbf{W}^\top\mathbf{W}(\theta_{S_m}^1-\theta_{S_m}^0)\right\}\\
    & = \frac{\|\beta_T^1-\beta_T^0\|_2^2}{16}\exp\left\{-\sum_{i=1}^d\left(\frac{C_T}{2b(\sigma_T)}(\beta_T^1-\beta_T^0)_i^2 + \sum_{m=1}^M \frac{\lambda_1^{(m)}C_{S_m}}{2a^{(m)}(\sigma_{S_m})}(\beta_{S_m}^1-\beta_{S_m}^0)_i^2  \right)\right\},
\end{align*}
where 
\[
\theta_T^j = E\beta_T^j, \theta_{S_m}^j = E\beta_{S_m}^j
\]
and $E\in\mathbb{R}^{d\times d}$ is any matrix that satisfies 
\[
E^\top \mathbf{W}^\top\mathbf{W}E = I .
\]
By Lemma~\ref{lemma:onedimconditions}, for any $K\geq\frac{1}{4}$, we can choose $\beta_T^0,\beta_T^1,\beta_{S_m}^0,\beta_{S_m}^1,m\in [M],$ such that 
\[
(\beta_T^0- \beta_{S_m}^0)_i^2\leq \frac{U_m^2}{d}, (\beta_T^1- \beta_{S_m}^1)_i^2\leq \frac{U_m^2}{d}, (\beta_T^0-\beta_T^1)_i^2 =\frac{1}{K} \frac{1}{\sum_{k=1}^M\frac{1}{\frac{U_k^2}{d} +\frac{a^{(k)}(\sigma_{S_k})}{C_{S_k}\lambda^{(k)}_1}}+\frac{C_T}{b(\sigma_T)}}
\]
and
\[
(\beta_{S_m}^0-\beta_{S_m}^1)_i^2\leq\frac{1}{K} \frac{\frac{1}{\frac{U_m^2}{d} +\frac{a^{(m)}(\sigma_{S_m})}{C_{S_m}\lambda^{(m)}_1}}\frac{a^{(m)}(\sigma_{S_m})}{C_{S_m}\lambda^{(m)}_1}}{\sum_{k=1}^M\frac{1}{\frac{U_k^2}{d} +\frac{a^{(k)}(\sigma_{S_k})}{C_{S_k}\lambda^{(k)}_1}}+\frac{C_T}{b(\sigma_T)}}.
\]
We therefore conclude
\begin{align*}
    \frac{\|\beta_T^1-\beta_T^0\|_2^2}{16}  & \exp\left\{-\sum_{i=1}^d \left(\frac{C_T}{2b(\sigma_T)}(\beta_T^1-\beta_T^0)_i^2 + \sum_{m=1}^M \frac{\lambda_1^{(m)}C_{S_m}}{2a^{(m)}(\sigma_{S_m})}(\beta_{S_m}^1-\beta_{S_m}^0)_i^2  \right)\right\}\\
    &\qquad\qquad\qquad\qquad \geq \frac{\exp\left(-\frac{d}{2K}\right)}{16K} \frac{d}{\sum_{k=1}^M\frac{1}{\frac{U_k^2}{d} +\frac{a^{(k)}(\sigma_{S_k})}{C_{S_k}\lambda^{(k)}_1}}+\frac{C_T}{b(\sigma_T)}}\\
    &\qquad\qquad\qquad\qquad \geq  \frac{\exp\left(-\frac{50}{K}\right)}{16K} \frac{d}{\sum_{k=1}^M\frac{1}{\frac{U_k^2}{d} +\frac{a^{(k)}(\sigma_{S_k})}{C_{S_k}\lambda^{(k)}_1}}+\frac{C_T}{b(\sigma_T)}}, \qquad\textrm{ for  } d\leq 100, \\
     &\qquad\qquad\qquad\qquad \geq  \frac{e^{-1}}{800} \frac{d}{\sum_{k=1}^M\frac{1}{\frac{U_k^2}{d} +\frac{a^{(k)}(\sigma_{S_k})}{C_{S_k}\lambda^{(k)}_1}}+\frac{C_T}{b(\sigma_T)}},
\end{align*}
where $K=50$ is chosen.

Now, for $d\geq 100$, we use Fano's method~\citep[Chapter~2]{ref:tsybakov2008introduction}. Let
\[
h  = \sqrt{\frac{1}{4K} \frac{1}{\sum_{k=1}^M\frac{1}{\frac{U_k^2}{d} +\frac{a^{(k)}(\sigma_{S_k})}{C_{S_k}\lambda^{(k)}_1}}+\frac{C_T}{b(\sigma_T)}}},
\]
and consider the hypercube
\[
\mathcal{C} = \{\beta\in\mathbb{R}^d: \beta_i \in \{-h,h\},i=1,\ldots,d\} .
\]
Then, by the Varshamov-Gilbert Lemma, since $d\geq8$, there exists a pruned hypercube $\beta_T^0,\ldots,\beta_T^J\in\mathcal{C}$ such that $J\geq 2^{d/8}$ and $H(\beta_T^j,\beta_T^k)\geq \frac{d}{8}$ for $0\leq j<k\leq J$, where $H$ denotes the Hamming distance, namely
\[
H(\beta_T^j,\beta_T^k) = \sum_{i=1}^d Id_{\{(\beta_T^j)_i\neq (\beta_T^k)_i\}} .
\]
We therefore have
\[
\min_{j\neq k} \|\beta_T^j-\beta_T^k\|_2^2\geq \frac{d}{8K} \frac{1}{\sum_{k=1}^M\frac{1}{\frac{U_k^2}{d} +\frac{a^{(k)}(\sigma_{S_k})}{C_{S_k}\lambda^{(k)}_1}}+\frac{C_T}{b(\sigma_T)}}.
\]
By Lemma~\ref{lemma:onedimconditions}, for any $m\in [M]$, there exists $k_m$ such that $0\leq k_m\leq h$ and
\[
(h-k_m)^2\leq \frac{U_m^2}{d}, (2k_m)^2\leq\frac{1}{K} \frac{\frac{1}{\frac{U_m^2}{d} +\frac{a^{(m)}(\sigma_{S_m})}{C_{S_m}\lambda^{(m)}_1}}\frac{a^{(m)}(\sigma_{S_m})}{C_{S_m}\lambda^{(m)}_1}}{\sum_{k=1}^M\frac{1}{\frac{U_k^2}{d} +\frac{a^{(k)}(\sigma_{S_k})}{C_{S_k}\lambda^{(k)}_1}}+\frac{C_T}{b(\sigma_T)}}.
\]
Hence, choosing $\beta_{S_m}^j,j = 0,\ldots, J,$ such that
\[
\forall 1\leq i \leq d, (\beta_{S_m}^j)_i =\begin{cases} k_m&\textrm{ if }(\beta_T^j)_i = h\\
-k_m&\textrm{ if }(\beta_T^j)_i = - h\end{cases},
\]
we obtain
\[
\|\beta_T^j - \beta_{S_m}^j \|_2^2 \leq U_m^2, \qquad j =0,\ldots,J, \qquad m\in [M] ,
\]
and
\[
 \mathrm{KL}((\theta_{S_1}^j,\ldots,\theta_{S_M}^j,\theta_T^j) ; (\theta_{S_1}^k,\ldots,\theta_{S_M}^k,\theta_T^k)) \leq   \frac{d}{2K}, \qquad\forall 0\leq j<k\leq M.
\]
We therefore have, by Fano's bound~\citep[Chapter~2]{ref:tsybakov2008introduction},
\begin{align*}
    \inf_{\widehat{\theta}}\sup_{D(\theta_{S_m},\theta_T)\leq U_m^2\forall m}\mathbb{E}[\ell(\widehat{\theta},\theta_T)]&\geq \frac{\min_{j\neq k} \|\beta_T^j-\beta_T^k\|_2^2}{4}\left(1-\frac{\frac{d}{2K}+\log 2}{\log J}\right)\\
    &\geq \frac{1}{32K}\left(1-\frac{4}{K\log2}-\frac{8}{d}\right)\frac{d}{\sum_{k=1}^M\frac{1}{\frac{U_k^2}{d} +\frac{a^{(k)}(\sigma_{S_k})}{C_{S_k}\lambda^{(k)}_1}}+\frac{C_T}{b(\sigma_T)}}\\
    &\geq \frac{3}{3200} \frac{d}{\sum_{k=1}^M\frac{1}{\frac{U_k^2}{d} +\frac{a^{(k)}(\sigma_{S_k})}{C_{S_k}\lambda^{(k)}_1}}+\frac{C_T}{b(\sigma_T)}},
\end{align*}
where $K=\frac{50}{3}$ is chosen and recalling that $d\geq100$. 
\end{proof}

\subsection{Proof of Remark~\ref{rm:leeandcourtadelb}}
\label{apdx:rm:leeandcourtadelb}
\textbf{Remark~\ref{rm:leeandcourtadelb}.}
\textit{
For the non-transfer learning setting considered in~\cite{ref:lee2020minimax}, our proof method gives rise to a lower bound of $$d\cdot b(\sigma_T)/C_T$$ which is sharper than their lower bound of
\[
\max\left\{\frac{\|\Lambda_\bW\|_1^2}{\|\Lambda_\bW\|_2^2},\lambda_{\textrm{min}}(\bW^\top\bW)\|\Lambda^{-1}_\bW\|_1\right\}\cdot b(\sigma_T)/C_T,
\]
where $\Lambda_\bW$ is the vector of eigenvalues of the positive-definite matrix $\bW^\top\bW$, and $\Lambda_\bW^{-1}$ denotes its coordinate-wise inverse.
}
\begin{proof}
By Holder's inequality, we have
\[
\|\Lambda_\bW\|_1^2\leq\|\Lambda_\bW\|_2^2 d.
\]
It is also readily verified that
\[
\lambda_{\textrm{min}}(\bW^\top\bW)\|\Lambda^{-1}_\bW\|_1\leq d,
\]
and the desired result follows.
\end{proof}

\subsection{Proof of Corollary~\ref{cor:multiplelr}}
\label{apdx:cor:multiplelr}
\textbf{Corollary~\ref{cor:multiplelr}.}
\textit{
When the GLMs considered are Gaussian 
LRMs,
then
a lower bound of the minimax risk is
\begin{equation*}
\inf_{\widehat{\theta}}\sup_{D(\theta_{S_m},\theta_T)\leq U_m^2, \forall m} \mathbb{E}[\ell(\widehat{\theta},\theta_T)]
\geq\frac{e^{-1}}{800} \frac{d}{\sum_{m=1}^M\frac{1}{\frac{U_m^2}{d} +\frac{\sigma_{S_m}^2}{\lambda_1^{(m)}}}+\frac{1}{\sigma_T^2}} .
\end{equation*}
}
\begin{proof}
Upon simply noting that, for Gaussian LRMs, we have $a^{(m)}(\sigma_{S_m}) = \sigma_{S_m}^2$, $b(\sigma_T)=\sigma_T^2$ and $C_{S_m}=C_T=1$,
the desired result then follows.
\end{proof}

\subsection{Auxiliary Result on the GLM Upper Bound}
\label{apx:glmupperboundproof}
We now provide details on the GLM upper bound in equation~\eqref{glmub}, where we additionally assume that
\[
\inf_z(\Psi^{(m)})^{''}(z)\geq L_{S_m}~\forall m, \qquad \inf_z\Gamma^{''}(z)\geq L_T .
\]
Consider the usual MLE estimator $\widehat\theta_{S_m}$ and $\widehat\theta_T$ for the source and target domains, and further consider a simplified interpolator
\[
\widehat\theta_{t} = \sum_{m=1}^M t_m \widehat\theta_{S_m} + t_{M+1} \widehat\theta_T,\qquad \sum_{m=1}^{M+1}t_m=1, \quad t_m\geq0.
\]
For any fixed admissible parameters satisfying $D(\theta_{S_m},\theta_T)\leq U_m^2, \forall m$, we claim that with probability at least $1-e^{-c}$, it holds that
\[
\ell(\widehat\theta_{t^\star},\theta_T)= (\widehat\theta_{t^\star} -\theta_{T})\mathbf{W}^\top\mathbf{W}(\widehat\theta_{t^\star} -\theta_{T})\leq \frac{d}{\sum_{m=1}^M \frac{1}{\frac{U_m^2}{d}+\frac{2C_{S_m}a^{(m)}(\sigma_{S_m})}{L_{S_m}^2\lambda_d^{(m)}}(c+\log(2dm))}+\frac{1}{\frac{2C_{T}b(\sigma_{T})}{L_{T}^2}(c+\log(2dm))}},
\]
where $t^\star$ solves
\[
\inf_{t = (t_1,\ldots,t_{M+1})\geq0,\sum_{m=1}^{M+1}t_m=1} \sum_{m=1}^M t_m^2\left(\frac{2dC_{S_m}a^{(m)}(\sigma_{S_m})}{L_{S_m}^2\lambda_d^{(m)}}(c+\log(2dm)) + U_m^2\right) + t_{M+1}^2 \frac{2dC_{T}b(\sigma_{T})}{L_{T}^2}(c+\log(2dm)).
\]
Given that the dimension and number of sources are fixed, we consider $c\gg\log(2dm)$ and compare against our lower bound in Theorem 3.1, namely
\[
\frac{e^{-1}}{800} \frac{d}{\sum_{m=1}^M\frac{1}{\frac{U_m^2}{d} +\frac{a^{(m)}(\sigma_{S_m})}{C_{S_m}\lambda_1^{(m)}}}+\frac{C_T}{b(\sigma_T)}},
\]
from which we find that there are gaps due to the ratios $\frac{C_{S_m}}{L_{S_m}}, \frac{C_{T}}{L_T}$ and the eigen gap $\frac{\lambda_1^{(m)}}{\lambda_d^{(m)}}$. 


\begin{proof}[Proof of Upper Bound] Using the sub-Gaussian concentration bound for GLM noise and the trick in Lemma 8 of Bastani (2021), we have the concentration bounds
\[
P\left((\widehat\theta_{S_m} -\theta_{S_m})\mathbf{W}^\top\mathbf{W}(\widehat\theta_{S_m} -\theta_{S_m})\leq \frac{2dC_{S_m}a^{(m)}(\sigma_{S_m})}{L_{S_m}^2\lambda_d^{(m)}}(c_1+\log(2d))\right) > 1 - e^{-c_1} , \qquad \forall m\in[M],
\]
and
\[
P\left((\widehat\theta_{T} -\theta_{T})\mathbf{W}^\top\mathbf{W}(\widehat\theta_{T} -\theta_{T})\leq \frac{2dC_{T}b(\sigma_{T})}{L_{T}^2}(c_1+\log(2d))\right) > 1 - e^{-c_1}.
\]
Due to independence, the probability of the intersection of the events happening is greater than $(1-e^{-c_1})^m\geq 1-2me^{-c_1}$ for large $c_1$. On the intersection of the events, we solve the problem
\[
\inf_{t = (t_1,\ldots,t_{M+1})\geq0,\sum_{m=1}^{M+1}t_m=1} \sum_{m=1}^M t_m^2\left(\frac{2dC_{S_m}a^{(m)}(\sigma_{S_m})}{L_{S_m}^2\lambda_d^{(m)}}(c_1+\log(2d)) + U_m^2\right) + t_{M+1}^2 \frac{2dC_{T}b(\sigma_{T})}{L_{T}^2}(c_1+\log(2d)),
\]
with optimal solution $t^\star$. We have that the optimal interpolator defined by
\[
\widehat\theta_{t^\star} = \sum_{m=1}^M t_m^\star \widehat\theta_{S_m} + t_{M+1}^\star \widehat\theta_T,
\]
satisfies
\[
(\widehat\theta_{t^\star} -\theta_{T})\mathbf{W}^\top\mathbf{W}(\widehat\theta_{t^\star} -\theta_{T})\leq \frac{d}{\sum_{m=1}^M \frac{1}{\frac{U_m^2}{d}+\frac{2C_{S_m}a^{(m)}(\sigma_{S_m})}{L_{S_m}^2\lambda_d^{(m)}}(c_1+\log(2d))}+\frac{1}{\frac{2C_{T}b(\sigma_{T})}{L_{T}^2}(c_1+\log(2d))}}.
\]
We choose $c_1 = c_1^{'} + \log(2m)$ so that this event happens with probability at least $1-e^{-c_1^{'}}$.
\end{proof}

\subsection{Auxiliary Result on the Comparison with~\cite{ref:Kalan2020}}
\label{apx:kalancomparison}
\cite{ref:Kalan2020} study a minimax lower bound for the linear regression setting (albeit under random design) and involve the spectral gap of the generalized eigenvalue problem we consider, with analogous definitions for the population distribution of their random design setting. The significance of our geometric perspective is best illustrated in comparison with their results where, in strong contrast, our analysis (in fixed design) takes care of the entire spectrum of the generalized eigenvalues. More precisely, our lower bound scaled by $1/n_T$ is lower bounded by (ignoring the constant $\exp{(-1/2)}/16$):
\[
\frac{1}{n_T}\sum_{i=1}^d \frac{1}{\frac{1}{\sigma^2} + \frac{1}{\alpha_i^\star U^2+  \frac{\sigma^2}{\lambda_i}}}\stackrel{\textrm{``spectral gap''}}{\geq}\frac{1}{n_T} \frac{d\sigma^2}{1 + \frac{1}{\frac{U^2}{d\sigma^2}+  \frac{1}{\lambda_1}}} \geq \frac{1}{n_T}\begin{cases}
c_1 \sigma^2 d & \textrm{ if } U^2\geq \tilde c_1\sigma^2d,\\
c_2 U^2 & \textrm{ if } \tilde c_2\frac{\sigma^2d}{1+\lambda_1}\leq U^2\leq \tilde c_1 \sigma^2d,\\
 c_3\frac{\sigma^2d}{1+\lambda_1} & \textrm{ if } U^2\leq \tilde c_2 \frac{\sigma^2d}{1+\lambda_1},
\end{cases}
\]
where $c_i,\tilde c_i$ are universal constants. The last expression essentially amounts to the lower bound in equation (3.1) in~\cite{ref:Kalan2020}, adjusting for the random designs and the scaling of $U$. We emphasize that the spectral gap may cause the last expression to be arbitrarily suboptimal, e.g., when $U=o(1)$, $\lambda_1\to\infty$ and $\lambda_d=O(1)$, while our analysis is sharp (i.e., the upper bound and lower bound match).

\section{Simulation Results}\label{apdx:results}
We first describe the choice of the algorithm from~\cite{ref:LiCaiLi2020} used for comparison against our proposed method. Since the setting is a single source domain from which learning is transferred to a target domain, we only consider Algorithm 1 from~\cite{ref:LiCaiLi2020}, i.e., the (original) Oracle Trans-Lasso algorithm.
Algorithm 4 from~\cite{ref:LiCaiLi2020} uses the $l_0$-norm to quantify the difference between the source and target parameters. However, for this algorithm, they require that the $l_0$-difference (denoted by $h_0$) is much smaller than the $l_0$ sparsity of the target parameter (denoted by $s$) for the learning to be effective, which does not hold in our simulation settings and in the real-world dataset (where $h_0=8$ and $s\approx8$ from Table~\ref{tab:comparisonrealdatadifferentnorms} and Figure~\ref{fig:targetregestimation} in Section~\ref{apdx:sim_realdata}). The Oracle Trans-Lasso algorithm from Appendix C.2 in~\cite{ref:LiCaiLi2020} uses the $l_q$-norm, $q\in(0,1)$, to quantify the difference between the source and target parameters. However, for this algorithm, they require that the $l_q$-difference (denoted by $h_q$) is much smaller than $\sqrt{s\log d/n_T}$ for the learning to be effective, which also does not hold in our simulation settings and in the real-world dataset (refer to Table~\ref{tab:comparisonrealdatadifferentnorms} in Section~\ref{apdx:sim_realdata} and note that $\sqrt{\log d/n_T}$ is between $0.018$ and $0.2$). We therefore choose to compare our proposed method with Algorithm 1 from~\cite{ref:LiCaiLi2020}, which uses the $l_1$-norm to quantify the difference.

\subsection{Additional Simulation Comparisons}
\label{apdx:sim}
As part of our additional simulation results related to the misspecification of $U$, we vary the parameter values from the baseline in Section~\ref{ssec:simmisspec} such that the rows of $\mathbf{X}$ are generated independently by a zero-mean Gaussian with a Toeplitz covariance matrix~\citep[Section~5.2]{ref:LiCaiLi2020},
or that both $\mathbf{X}$ and $\mathbf{W}$ are generated in this way.
The corresponding results are provided in the left plot and right plot of Figure~\ref{fig:mistoeplitz}, respectively.
We observe from these results that the introduction of correlation, either in only $\mathbf{X}$ or in both $\mathbf{X}$ and $\mathbf{W}$, does not impact the performance of either method when compared to the uncorrelated case in the left plot of Figure~\ref{fig:misbaseandlargeu}.
\begin{figure}[htb]
    \centering
    \includegraphics[width=0.7\columnwidth]{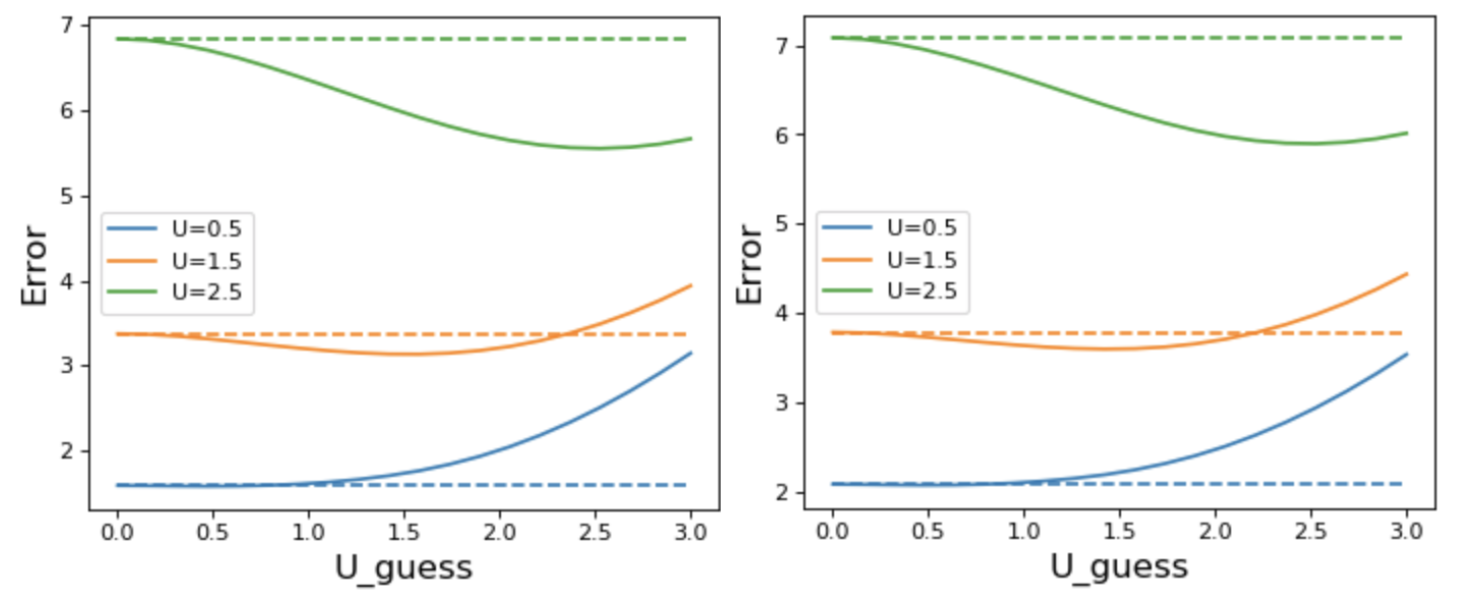}
    \caption{Simulation Results with Toeplitz Covariance Matrix for Designs.
    Solid Lines (Dashed Lines) Represent the Proposed Method (Basic Pooling Method).
    }
    \label{fig:mistoeplitz}
\end{figure}

We also consider varying the noise variances.  In particular, the noise variances are changed to $\sigma_S^2=1,\sigma_T^2=5$ or $\sigma_S^2=5,\sigma_T^2 =1$, the results of which are provided in the 
left plot and right plot of Figure~\ref{fig:misunequalvar}, respectively. The proposed method handles large variances in the source data much better than the basic pooling method, while high variance in (smaller sized) target data has no significant impact on either method.
\begin{figure}[h]
    \centering
    \includegraphics[width=0.7\columnwidth]{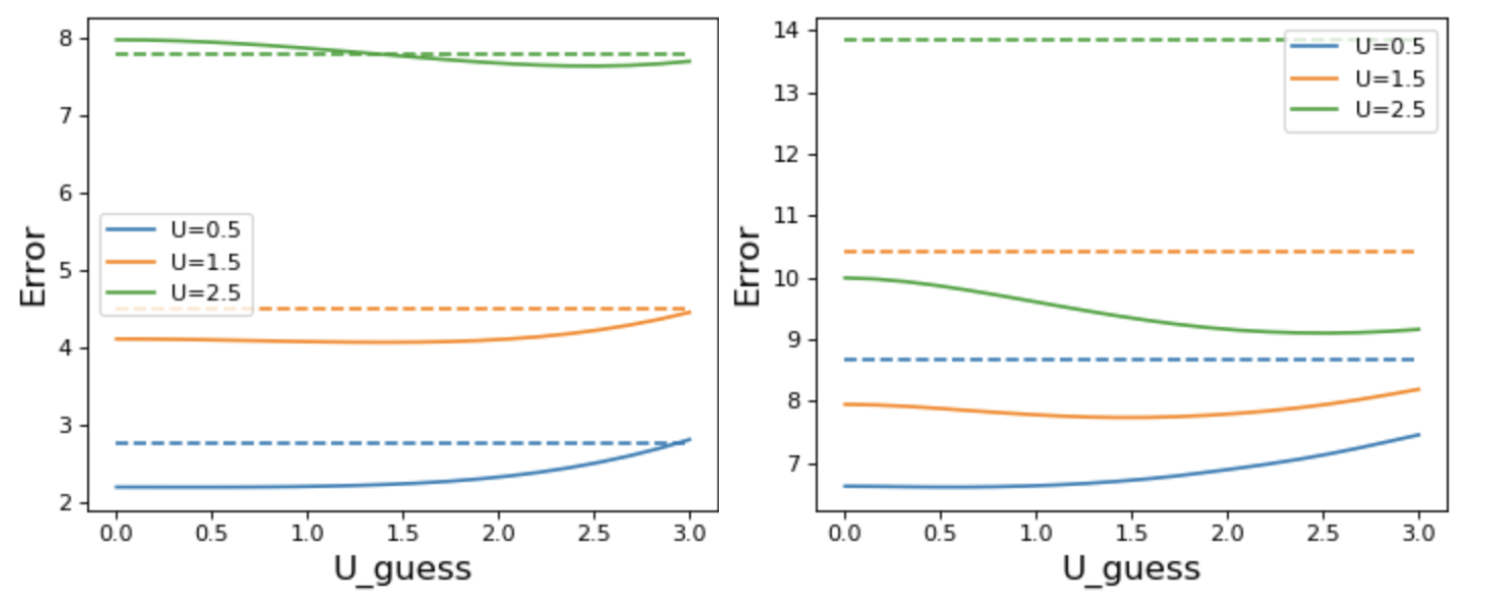}
    \caption{Simulation Results with Unequal Noise Variances. Solid Lines Represent the Proposed Method. Dashed Lines Represent the Basic Pooling Method.}
    \label{fig:misunequalvar}
\end{figure}

We further consider varying the different dimensions.
In particular, the dimension $d$ is changed to $5$ or $d=100$ with $\beta_T$ three-sparse; 
specifically, 
the first three elements of $\beta_T$ are one and the rest are zero.
The corresponding results are provided in the left plot and right plot of Figure~\ref{fig:misdimchange}, respectively.
Lower dimensionality (left) seems to improve the performance gap between the two methods as compared to the left plot of Figure~\ref{fig:misbaseandlargeu}, while this gap shortens in high-dimensions with extreme sparsity (right).
\begin{figure}[h]
    \centering
    \includegraphics[width=0.7\columnwidth]{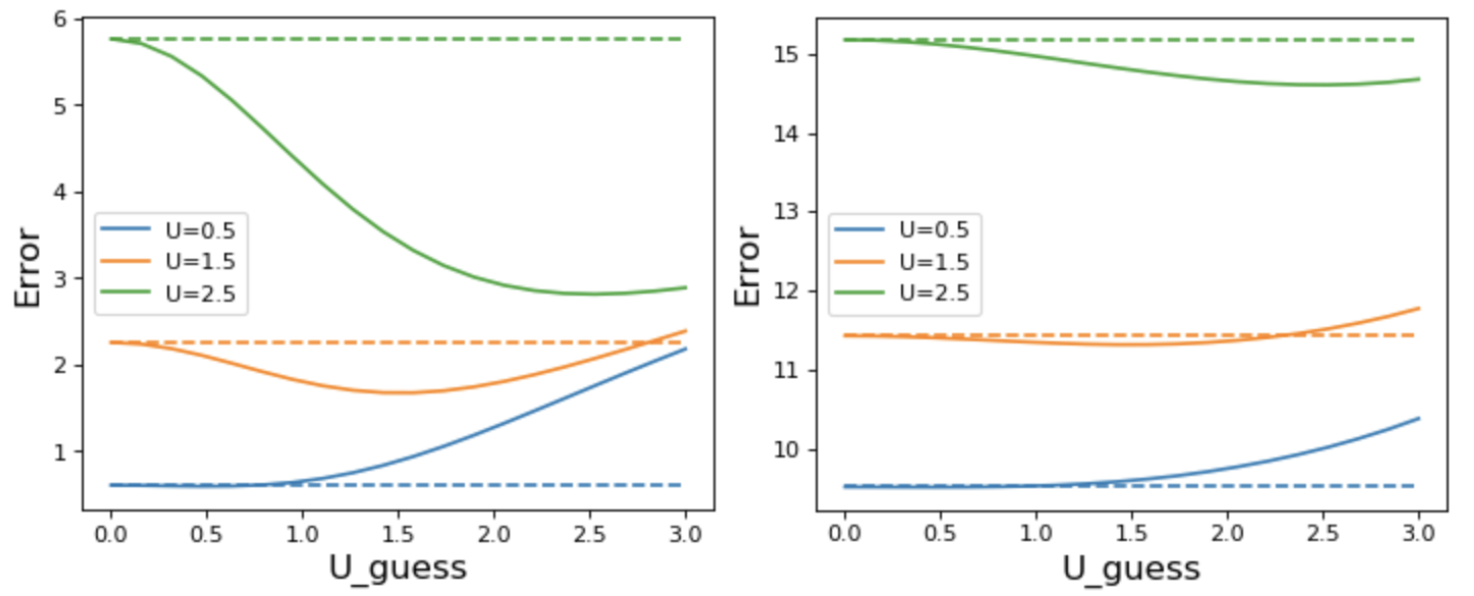}
    \caption{Simulation Results with Different Dimensions. Solid Lines Represent the Proposed Method. Dashed Lines Represent the Basic Pooling Method.}
    \label{fig:misdimchange}
\end{figure}

We additionally consider varying the magnitude (sup-norm) of the ground truth value of the parameter $\beta_T$ from the baseline in Section~\ref{ssec:simmisspec} such that $\beta_T$ is the vector whose component values are all $0.1$, or that $\beta_T$ is the vector whose component values are all $10$.  The corresponding results are provided in the left plot and right plot of Figure~\ref{fig:misbetamag}, respectively. We observe that the plots remain exactly the same as the left plot of Figure~\ref{fig:misbaseandlargeu}, which is consistent with Theorem~\ref{prop:UB:LRM} that the worst-case performance of our method does not depend on the magnitude of $\beta_T$, and similarly for the basic pooling method.
\begin{figure}[htb]
    \centering
    \includegraphics[width=0.7\columnwidth]{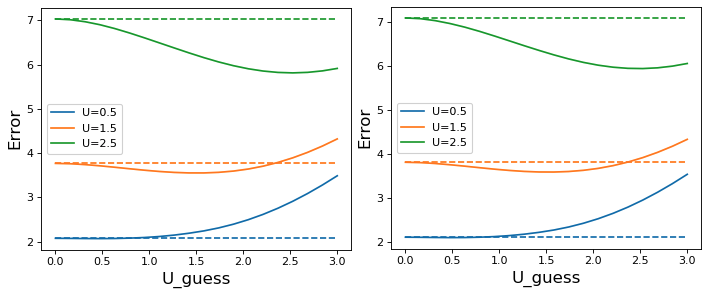}
    \caption{Simulation Results with Different Magnitudes of the Ground Truth $\beta_T$.
    Solid Lines (Dashed Lines) Represent the Proposed Method (Basic Pooling Method).
    }
    \label{fig:misbetamag}
\end{figure}

We next investigate the behavior of the performance gap as the sample size increases. In particular, we vary the sample sizes to be $n_S= 10000$ and $n_T=1000$, and present the corresponding results in Figure~\ref{fig:missamplechange}. We observe from these results that the plot remains the same as the left plot of Figure~\ref{fig:misbaseandlargeu}, demonstrating the robustness of the performance gap.

\begin{figure}[htb]
    \centering
    \includegraphics[width=0.38\columnwidth]{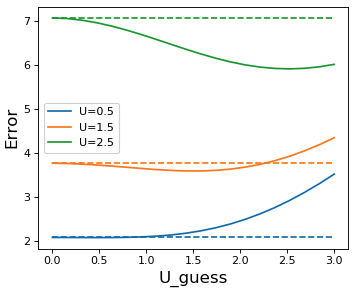}
    \caption{Simulation Results with A Larger Sample Size.
    Solid Lines (Dashed Lines) Represent the Proposed Method (Basic Pooling Method).
    }
    \label{fig:missamplechange}
\end{figure}

As part of our additional simulation results related to the comparisons against competing methods in the literature, we consider the same settings in Section~\ref{sec:moderatecomparison} and Section~\ref{ssec:highcomparison} with a smaller $n_S/n_T$ ratio. In particular, we vary the sample sizes to be $n_S=200$ and $n_T=100$.  For the methods under consideration, we report the average estimation error of $\theta_T$ in~\eqref{eq:lossfunction}, and its standard deviation, over $1000$ simulation runs. The results are summarized in the left-half of Table~\ref{tab:comparisonbothsmallerns} for moderate-dimensions and the right-half of Table~\ref{tab:comparisonbothsmallerns} for high-dimensions. We observe behaviors of our proposed method relative to the competing methods to be consistent with those in Section~\ref{sec:moderatecomparison} and Section~\ref{ssec:highcomparison}.

\begin{table*}[tb]
\caption{Simulation Results Comparing the Proposed Method to Other Competing Methods in Moderate-Dimensions (Left-Half) and in High-Dimensions (Right-Half).
``Basic'' Represents the Lowest of the Errors Attained by the Three Basic Methods in Section~\ref{ssec:compare}. Numbers in Parentheses Are Standard Deviations.}
\label{tab:comparisonbothsmallerns}
\begin{center}
\begin{tabular}{|c|c|c|c|c||c|c|c|c|}
\hline
$U$ &Basic&Proposed&Two-Step&Trans-Lasso &Basic&Proposed&Two-Step&Trans-Lasso \\
\hline
0.5 & 6.8(2.1) &  7.1(2.2) & 10.6(3.6) & \textbf{1.1}(2.4) & 33.9(7.2) & \textbf{30.9}(7.6) & 61.1(23.7) & 34.2(7.5)\\
1.5 & 8.4(2.5) & 8.5(2.7) & 11.4(3.5) & \textbf{1.7}(2.2) & 36.3(6.2) & \textbf{32.6}(6.6) & 62.9(22.2) & 37.4(6.6)\\
2.5 & 10.5(2.7) & 9.9(2.5) &  14.1(4.5) & \textbf{2.5}(2.6) & 38.6(5.9)& \textbf{34.3}(6.1) & 69.8(22.1) & 40.5(6.3)\\
3.5 & 13.1(3.2) & 11.6(3.1) & 15.5(4.6) & \textbf{4.8}(3.4) & 44.6(5.6) & \textbf{38.3}(6.4) & 77.1(23.7) & 47.6(6.0)\\
4.5 & 17.5(3.7) & 14.5(3.9) & 17.6(6.2) & \textbf{8.7}(4.5) & 50.8(6.1) & \textbf{41.5}(5.8) & 84.8(23.5) & 55.2(7.4)\\
5.5 & 19.3(5.8) & 15.3(4.8) & 18.7(7.2) & \textbf{11.7}(4.9) & 58.3(7.2) & \textbf{45.4}(8.1) & 98.8(24.9) & 63.6(8.1)\\
6.5 & 19.3(6.4) & 17.3(6.0) & 19.7(7.7) & \textbf{16.6}(6.6) & 55.9(12.5) & \textbf{49.3}(8.7) & 115.9(29.4) & 74.4(7.5)\\
7.5 & 19.4(6.8) & \textbf{18.3}(6.4) & 19.3(7.5) & 20.4(7.2) & 57.1(11.6) & \textbf{52.9}(9.4) & 130.3(30.1) & 87.3(9.8)\\
8.5 & 20.1(5.5) & \textbf{18.5}(6.1) & 19.0(6.4) & 21.7(9.4) & 56.3(12.8) & \textbf{53.4}(10.7) & 151.8(28.3) & 98.9(9.8)\\
9.5 & 21.4(12.4) & \textbf{18.8}(6.9) & 20.5(7.6) & 23.2(9.0) & 57.7(13.9) & \textbf{56.9}(9.6) & 167.6(34.5) & 112.5(10.8)\\
\hline
\end{tabular}
\end{center}
\end{table*}

\subsection{Real-World Dataset Experiments}
\label{apdx:sim_realdata}
As part of our final set of empirical results considered in Section~\ref{ssec:realdatacomparison}, we compare our proposed method against the other competing methods using the Uber$\&$Lyft dataset (https://www.kaggle.com/brllrb/uber-and-lyft-dataset-boston-ma) of Uber and Lyft cab rides collected in Boston, MA. 
Recall that we consider UberX to be the source model and standard Lyft service to be the target model where the learning problem comprises prediction of the price using $d=32$ numerical features.
Given that the focus of our study is on the benefit of transfer learning, we restrict our experiments to small random subsamples and we summarize in Table~\ref{tab:comparisonrealdata} the results taken over $100$ repeated independent experiments.

The Uber$\&$Lyft dataset consists of 55094 observations for the source and 51235 observations for the target, from which we obtain and present in Figures~\ref{fig:sourceregestimation} and~\ref{fig:targetregestimation} the corresponding ground-truth regression parameters as bar plots for the source and target models, respectively.
We observe that the parameters are moderately sparse with the feature ``surge\_multiplier'' having much larger magnitudes in comparison with the other features.
Figure~\ref{fig:targetvssourceregestimation} plots the difference between these two parameters, from which we again observe that the difference is sparse. We also compute the $l_q$-distance, $q\in\{0,0.5,1\}$, between the source and target parameters and summarize these results in Table~\ref{tab:comparisonrealdatadifferentnorms}.

The results in Table~\ref{tab:comparisonrealdata} show that our proposed method attains a better performance on average, by a small margin relative to the basic methods and by a large margin relative to the two-step estimator and Trans-Lasso. We note that the $l_q$ sparsity, $q\in[0,1]$, required by the last two methods does not reasonably capture the contrast between the source and target models of the real-world dataset, due to the moderate dimensions and the existence of one dominating feature
In particular, while \cite{ref:bastani2021predicting} shows that the two-step joint estimator performs well when the difference of regression parameters is $l_0$ sparse, their result applies to high-dimensions which is in contrast to the $32$ dimensions of the dataset at hand.
This helps to explain why their two-step joint estimator does not yield good performance in Table~\ref{tab:comparisonrealdata} where the moderate dimensions and a single feature of ``surge\_multiplier'' significantly affects the model. 
\cite{ref:LiCaiLi2020} show that their Trans-Lasso method performs well when the $l_1$-difference of the regression parameters (denoted by $h_1$) satisfies $h_1 \ll s\sqrt{\log d/n_T}$,  where $s$ is the $l_0$ sparsity of the target regression parameters.
However, for the real-world dataset, we observe that $s \approx 8$ (see Figure~\ref{fig:targetregestimation}), that the factor $\sqrt{\log(d)/n_T}$ is smaller than $0.2$,
that $h_q$ (i.e., $l_q$-difference of the regression parameters) increases as $q$ decreases from $1$ toward $0$, and that $h_q$ presents a discontinuity at $q=0$ with $h_0=8$, some of which is illustrated in Table~\ref{tab:comparisonrealdatadifferentnorms}.
Hence, the $l_1$ (or $l_q$ in general) relationship assumed by \cite{ref:LiCaiLi2020} does not appear to hold for the dataset at hand. 
This in turn helps to explain why Trans-Lasso does not yield good performance in Table~\ref{tab:comparisonrealdata}.

\begin{figure}[htb]
    \centering
    \includegraphics[width=0.9\columnwidth]{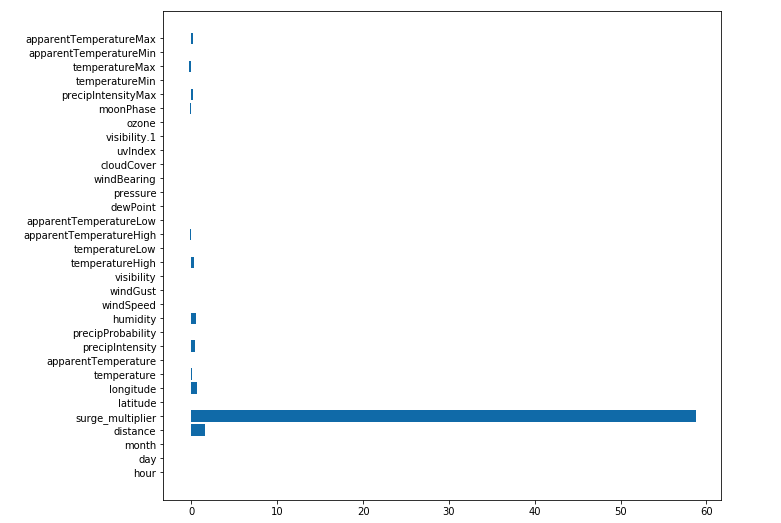}
    \caption{Estimated Regression Parameters of Source Model From the Entire Source Dataset.
    }
    \label{fig:sourceregestimation}
\end{figure}

\begin{figure}[htb]
    \centering
    \includegraphics[width=0.9\columnwidth]{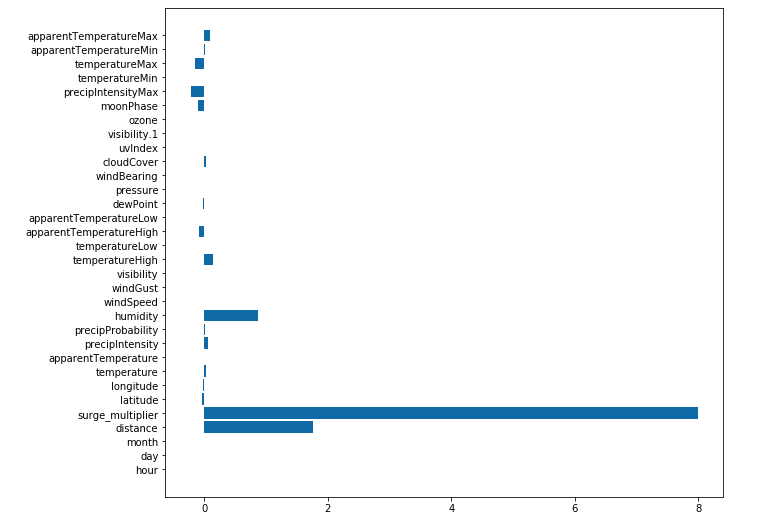}
    \caption{Estimated Regression Parameters of Target Model From the Entire Target Dataset.
    }
    \label{fig:targetregestimation}
\end{figure}

\begin{figure}[htb]
    \centering
    \includegraphics[width=0.83\columnwidth]{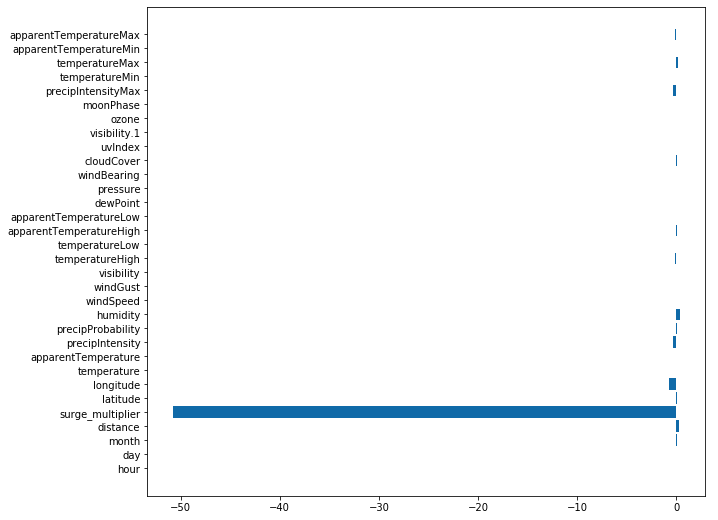}
    \caption{Difference in Regression Parameters of Target versus Source Model From the Entire Dataset.
    }
    \label{fig:targetvssourceregestimation}
\end{figure}

\begin{table}[h]
\caption{Results For Distance (under Different Sparsity Norms) between Source and Target Ground-truth Parameters on Uber$\&$Lyft Data. The $l_0$-distance Has a Threshold of $0.1$ to Determine Non-zero Entries.} \label{tab:comparisonrealdatadifferentnorms}
\begin{center}
\begin{tabular}{|c|c|c|}
\hline
$l_0$-distance ($h_0$) &$l_{0.5}$-distance ($h_{0.5}$) &$l_1$-distance ($h_{1}$)\\
\hline
8 &  189.67 & 53.65\\
\hline
\end{tabular}
\end{center}
\end{table}

\end{document}